\documentclass[12pt]{article}
\usepackage{natbib}
\usepackage{epsfig}
\usepackage{natbib}
\usepackage{times}
\usepackage{amsmath,amsthm,amsfonts,amssymb,mathrsfs}

\newcommand{\bproof}{\begin{proof}}
\newcommand{\eproof}{\end{proof}}

\DeclareMathOperator\trace{tr}

\begin{document}
\bibliographystyle{plain}
\makeatletter

\renewcommand{\theequation}{\thesection.\arabic{equation}}
\numberwithin{equation}{section}
\renewcommand{\bar}{\overline}
\newtheorem{conjecture}{Conjecture}[section]
\newtheorem{theorem}{Theorem}[section]
\newtheorem{question}{Question}[section]
\newtheorem{proposition}{Proposition}[section]
\newtheorem{lemma}{Lemma}[section]
\newtheorem{corollary}{Corollary}[section]
\newtheorem{definition}{Definition}[section]
\newtheorem{algorithm}{Algorithm}[section]
\newtheorem{problem}{\em Problem}[section]
\newtheorem{remark}{Remark}[section]
\newtheorem{example}{Example}[section]
\newtheorem{case}{Case}[section]
\newtheorem{assumption}{Assumption}[section]

\renewcommand\proofname{\bf Proof} 
\def\eop{$\rule{1.3ex}{1.3ex}$}
\renewcommand\qedsymbol\eop  
\numberwithin{equation}{section}
\makeatletter

\newcommand{\hlam}{{\hat \lambda}}
\newcommand{\lh}{{\hat \lambda}}
\newcommand{\blam}{{\bar \lambda}}
\newcommand{\hS}{{\hat S}}
\newcommand{\zR}{{\R \backslash \{0\}}}
\renewcommand{\bar}{\overline}
\newcommand{\U}{{\cal U}}
\newcommand{\zRn}{{(\R \backslash \{0\})^n}}
\newcommand{\beq}{\begin{equation}}
\newcommand{\eeq}{\end{equation}} 
\newcommand{\bz}{{\bf z}} \newcommand{\bx}{{\bf x}}
\newcommand{\bt}{{\bf t}} \newcommand{\bi}{\begin{itemize}}
\newcommand{\be}{\begin{enumerate}} \newcommand{\ei}{\end{itemize}}
\newcommand{\ee}{\end{enumerate}} \newcommand{\calH}{{\cal H}}
\newcommand{\E}{{\cal E}} \newcommand{\J}{{\cal J}}
\newcommand{\Ga}{{\Gamma}} \newcommand{\ga}{{\gamma}}
\newcommand{\Om}{{\Omega}} \newcommand{\om}{{\omega}}
\newcommand{\gam}{{\gamma}} \newcommand{\La}{{\Lambda}}
\newcommand{\hbeta}{{\hat \beta}} \newcommand{\Gam}{{\Gamma}}
\newcommand{\homega}{{\hat \om}} \newcommand{\bS}{{\mathbb S}}
\newcommand{\R}{{\mathbb R}} \newcommand{\N}{{\mathbb N}}
\newcommand{\calE}{{\cal E}} \newcommand{\calG}{{\cal G}}
\newcommand{\calV}{{\cal V}} \newcommand{\calK}{{\cal K}}
\newcommand{\calN}{{\cal N}} \newcommand{\calU}{{\cal U}}
\newcommand{\calT}{{\cal T}} \newcommand{\calY}{{\cal Y}}
\newcommand{\calO}{{\cal O}} \newcommand{\calX}{{\cal X}}
\newcommand{\calW}{{\cal W}} \newcommand{\calI}{{\cal I}}
\newcommand{\W}{{\cal W}} \newcommand{\G}{{\cal G}}
\newcommand{\K}{{\cal K}} \newcommand{\OO}{{\bf O}}
\newcommand{\hK}{{\hat K}} \newcommand{\X}{{\cal X}}
\newcommand{\M}{{\cal M}} \newcommand{\KG}{{\cal K(\calG)}}
\newcommand{\lam}{{\lambda}} \newcommand{\tlam}{{\tilde{\lambda}}}
\newcommand{\calM}{{\cal M}} \newcommand{\calA}{{\cal A}}
\newcommand{\calB}{{\cal B}} \newcommand{\calL}{{\cal L}}
\newcommand{\calD}{{\cal D}} \newcommand{\calR}{{\cal R}}
\newcommand{\pp}{{\cal P}} \newcommand{\hc}{{\hat c}}
\newcommand{\ck}{{c_K}} \newcommand{\hL}{{\hat L}}
\newcommand{\tL}{{\bar L}} \newcommand{\sK}{{\SSS K}}
\newcommand{\hg}{{g_K}} \newcommand{\tf}{{f_K}}
\newcommand{\hy}{{y_K}} \newcommand{\haty}{{\hat y}}
\newcommand{\hG}{{\hat \Gam}} \newcommand{\vt}{{\vec t}}
\newcommand{\vv}{{\vec v}} \newcommand{\lb}{{\langle}}
\newcommand{\rb}{{\rangle}} \newcommand{\by}{{\bf y}}
\newcommand{\btau}{{\bf \tau}} \newcommand{\bu}{{\bf u}}
\newcommand{\bv}{{\bf v}} \newcommand{\tby}{\tilde{{\bf y}}}
\newcommand{\Sb}{{\bf S}} \newcommand{\Mb}{{\bf M}}
\newcommand{\Ob}{{\bf O}} \newcommand{\SSS}{\scriptscriptstyle}
\def\boldf#1{\hbox{\rlap{$#1$}\kern.4pt{$#1$}}}
\newcommand{\balpha}{{\boldf \alpha}} \newcommand{\wh}{\hat w}
\newcommand{\Wh}{\hat W} \newcommand{\wb}{\bar w}
\newcommand{\Wb}{\bar W} \newcommand{\xb}{\bar x}
\newcommand{\cb}{\bar c} \newcommand{\trans}{^{\scriptscriptstyle
\top}} \newcommand{\tW}{\tilde{W}} \newcommand{\tw}{\tilde{w}}
\newcommand{\tbeta}{{\tilde \beta}}
\newcommand{\betak}{{\beta^{(k)}}}

\newcommand{\figsheight}{4.0cm}
\renewcommand\baselinestretch{1}


\begin{titlepage}
\advance\topmargin by 0.5in
\begin{center}

\vspace{.5truecm}
{\Large Regularizers for Structured Sparsity}
\vspace{.88truecm}

\end{center}

\begin{center}

{\bf Charles A. Micchelli}$^{(1),(2)}$
\\ \vspace{.2truecm}
{\bf Jean M. Morales}$^{(3)}$
\\ \vspace{.2truecm}
{\bf Massimiliano Pontil}$^{(3)}$

\vspace{.8truecm}

\noindent (1)
Department of Mathematics \\
City University of Hong Kong \\
83 Tat Chee Avenue, Kowloon Tong \\
Hong Kong \\
\vspace{.35truecm}
(2)
Department of Mathematics and Statistics \\
State University of New York \\
The University at Albany \\
1400 Washington Avenue \\
Albany, NY, 12222, USA

\vspace{.35truecm}

\noindent (3) Department of Computer Science \\
University College London \\
Gower Street, London WC1E \\ England, UK \\
E-mail: {\em \{m.pontil,j.morales\}@cs.ucl.ac.uk}

\vspace{.65truecm}


\end{center}

\begin{abstract}
\noindent 
We study the problem of learning a sparse linear
regression vector under additional conditions on the structure of
its sparsity pattern. This problem is relevant in machine learning,
statistics and signal processing. It is well known that a linear
regression can benefit from knowledge that the underlying
regression vector is sparse. The combinatorial problem of
selecting the nonzero components of this vector can be ``relaxed'' 
by regularizing the squared error with a convex penalty function
like the $\ell_1$ norm. However, in many applications, additional
conditions on the structure of the regression vector and its
sparsity pattern are available. Incorporating this information
into the learning method may lead to a significant decrease of
the estimation error.

In this paper, we present a family of convex penalty functions,
which encode prior knowledge on the structure of the vector 
formed by the absolute values of the regression coefficients. 
This family subsumes the $\ell_1$ norm and is flexible
enough to include different models of sparsity patterns, which are
of practical and theoretical importance. We establish the basic
properties of these penalty functions and discuss some examples
where they can be computed explicitly. Moreover, we present a
convergent optimization algorithm for solving regularized least
squares with these penalty functions. Numerical simulations
highlight the benefit of structured sparsity and the advantage
offered by our approach over the Lasso method and other related
methods.
\end{abstract}
\end{titlepage}


\section{Introduction}
The problem of sparse estimation is becoming increasing important
in statistics, machine learning and signal processing.  In its
simplest form, this problem consists in estimating a regression
vector $\beta^* \in \R^n$ from a set of linear measurements $y \in
\R^m$, obtained from the model \beq y = X \beta^* + \xi
\label{eq:model} \eeq where $X$ is an $m \times n$ matrix, which
may be fixed or randomly chosen and $\xi \in \R^m$ is a vector
which results from the presence of noise.

An important rational for sparse estimation comes from the observation
that in many practical applications the number of parameters $n$ is
much larger than the data size $m$, but the vector $\beta^*$ is known
to be sparse, that is, most of its components are equal to zero. Under
this sparsity assumption and certain conditions on the data matrix $X$, 
it has been shown that regularization with the
$\ell_1$ norm, commonly referred to as the Lasso method \cite{lasso}, provides an
effective means to estimate the underlying regression vector, see for
example \cite{BRT,bunea2007soi,lounici2008snc,vandegeer2008hdg} and
references therein. Moreover, this method can reliably select the sparsity pattern
of $\beta^*$ \cite{lounici2008snc}, hence providing a valuable tool
for feature selection.


In this paper, we are interested in sparse estimation under
additional conditions on the sparsity pattern of the vector
$\beta^*$.  In other words, not only do we expect this vector to
be sparse but also that it is {\em structured sparse}, namely
certain configurations of its nonzero components are to be
preferred to others. This problem arises is several applications,
ranging from functional magnetic resonance imaging
\cite{gramfort2009,xiang2009}, to scene recognition in vision
\cite{harzallah2009}, to multi-task learning
\cite{AEP,kim09,oboz09} and to bioinformatics \cite{rapaport2008}, 
see \cite{Jenatton} for a discussion.

The prior knowledge that we consider in this paper is that the
vector $|\beta^*|$, whose components are the absolute value of the
corresponding components of $\beta^*$, should belong to some
prescribed convex subset $\La$ of the positive orthant. For
certain choices of $\La$ this implies a constraint on the sparsity
pattern as well. For example, the set $\La$ may include vectors
with some desired monotonicity constraints, or other constraints
on the ``shape'' of the regression vector. 
Unfortunately, the constraint that $|\beta^*| \in \La$ is nonconvex
and its implementation is computational challenging. To overcome this
difficulty, we propose a family of penalty functions, which are based
on an extension of the $\ell_1$ norm used by the Lasso method and
involves the solution of a smooth convex optimization problem. These
penalty functions favor regression vectors $\beta$ such that $|\beta| \in \Lambda$, thereby 
incorporating the structured sparsity constraints.

Precisely, we propose to estimate $\beta^*$ as a solution of the
convex optimization problem
\beq\min \left\{\|X\beta-y\|^2_2 + 2\rho \Omega(\beta|\La) : \beta \in \R^n\right\}\label{eq:method1}\eeq
where $\|\cdot\|_2$ denotes the Euclidean norm, $\rho$ is a positive parameter and the penalty function takes the form
\beq
\Omega(\beta|\La) = \inf \left\{ \frac{1}{2}\sum_{i \in \N_n} \left(\frac{\beta_i^2}{\lambda_i} + \lambda_i\right): \lam \in \La\right\}.
\nonumber
\eeq

As we shall see, a key property of the penalty function is that it
exceeds the $\ell_1$ norm of $\beta$ when $|\beta| \notin \La$, and it
coincides with the $\ell_1$ norm otherwise.
This observation suggests a heuristic interpretation of the
method \eqref{eq:method1}: among all vectors $\beta$ which have a fixed
value of the $\ell_1$ norm, the penalty function $\Omega$ will
encourage those for which $|\beta| \in \La$. Moreover, when
$|\beta| \in \La$ the function $\Omega$ reduces to the $\ell_1$ norm and, so,
the solution of problem $\eqref{eq:method1}$ is expected to be
sparse. The penalty function therefore will encourage certain desired
sparsity patterns. Indeed, the sparsity pattern of $\beta$
is contained in that of the auxiliary vector $\lam$ at the optimum and, so, 
if the set $\La$ allows only for certain sparsity patterns of $\lambda$, the same
property will be ``transferred'' to the regression vector $\beta$. 

There has been some recent research interest on structured
sparsity, see \cite{huang2009,jacob,Jenatton,lounici2010oracle,mosci2010,Yuan09,group_lasso} and
references therein. Closest to our approach are penalty methods
built around the idea of mixed $\ell_1$-$\ell_2$ norms. In
particular, the group Lasso method \cite{group_lasso} assumes that
the components of the underlying regression vector $\beta^*$ can
be partitioned into prescribed groups, such that the restriction
of $\beta^*$ to a group is equal to zero for most of the groups.
This idea has been extended in \cite{Jenatton,binyu} by 
considering the possibility that the groups overlap according to
certain hierarchical or spatially related structures.
Although these methods have proved valuable in applications, they have
the limitation that they can only handle more restrictive classes of
sparsity, for example patterns forming only a single connected region.
Our point of view is different from theirs and provides a means to
designing more flexible penalty functions which maintain convexity
while modeling richer model structures. For example, we will
demonstrate that our family of penalty functions can model sparsity
patterns forming multiple connected regions of coefficients.

The paper is organized in the following manner. In Section \ref{sec:2}
we establish some important properties of the penalty function. In
Section \ref{sec:box} we address the case in which the set $\La$ is a box.
In Section \ref{sec:WP} we derive the form of the penalty function
corresponding to the wedge with decreasing coordinates and in Section
\ref{sec:GP} we extends this analysis to the case in which the
constraint set $\La$ is constructed from a directed graph. In Section
\ref{sec:comp} we discuss useful duality relations and in Section
\ref{sec:algo} we address the issue of solving the problem
\eqref{eq:method1} numerically by means of an alternating
minimization algorithm. Finally, in Section \ref{sec:exp} we provide
numerical simulations with this method, showing the advantage offered
by our approach.

A preliminary version of this paper appeared in the proceedings of the
Twenty-Fourth Annual Conference on Neural Information Processing Systems (NIPS 2010) 
\cite{MMP-10}. The new version contains Propositions \ref{prop:0},~\ref{prop:comb} and \ref{prop:dual}, 
the description of the graph penalty in Section \ref{sec:GP}, Section \ref{sec:comp}, 
a complete proof of Theorem \ref{thm:aa} and an experimental comparison with the method of \cite{huang2009}.


\section{Penalty function}
\label{sec:2}
In this section, we provide some general comments on the penalty
function which we study in this paper. 

We first review our notation. We denote with $\R_+$ and $\R_{++}$ the nonnegative 
and positive real line, respectively. For every $\beta \in \R^n$ we define $|\beta| \in
\R_+^n$ to be the vector formed by the absolute values of the components of $\beta$, that is, 
$|\beta| = (|\beta_i|: i \in \N_n)$, where $\N_n$ is the set of
positive integers up to and including $n$. 
Finally, we define the $\ell_1$ norm of vector $\beta$ as $\|\beta\|_1 = \sum_{i\in \N_n} |\beta_i|$ 
and the $\ell_2$ norm as $\|\beta\|_2 = \sqrt{\sum_{i \in \N_n} \beta_i^2}$. 

Given an $m \times n$ input data matrix $X$ and an output vector $y \in \R^m$, obtained
from the linear regression model $y=X\beta^*+\xi$ discussed earlier, we consider
the convex optimization problem
\beq
\inf \left\{\|X\beta-y\|^2_2 +
2\rho\, \Gamma(\beta,\lam)
: \beta \in \R^n, \lambda\in\La
 \right\}
\label{eq:primal}
\eeq
where $\rho$ is a positive parameter, $\La$ is a prescribed convex
subset of the positive orthant $\R^n_{++}$ and the function $\Gamma :
\R^n \times \R_{++}^n \rightarrow \R$ is given by the formula
$$\Gamma(\beta,\lam) = \frac{1}{2}\sum_{i\in \N_n}
\left(\frac{\beta_i^2}{\lambda_i} + \lambda_i\right).$$
Note that in \eqref{eq:primal}, for a fixed $\beta \in \mathbb{R}^n$,
the infimum over $\lambda=(\lam_i: i \in \N_n)$ in general is not
attained, however, for a fixed $\lambda \in \La$, the infimum over
$\beta$ is always attained.

Since the auxiliary vector $\lam$ appears only in the second term of the objective function 
of problem \eqref{eq:primal}, 
and our goal is to estimate $\beta^*$, we may also directly consider
the regularization problem
\beq
\min \left\{\|X\beta-y\|^2_2 + 2\rho\, \Omega(\beta|\La) :
\beta \in \R^n\right\},
\label{eq:method}
\eeq
where the penalty function takes the form
\beq
\Omega(\beta|\La) = \inf \left\{\Gamma(\beta,\lam): \lam \in \La\right\}.
\label{eq:gen-omega}
\eeq
Note that $\Gamma$ is convex on its domain because each of its
summands are likewise convex functions. Hence, when the set $\La$ is
convex it follows that $\Omega(\cdot|\La)$ is a convex function and
\eqref{eq:method} is a convex optimization problem.

An essential idea behind our construction of the penalty function is
that, for every $\lam \in \R_{++}$, the quadratic function
$\Gamma(\cdot,\lam)$ provides a smooth approximation to $|\beta|$
from above, which is exact at $\beta= \pm\lam$. We indicate this
graphically in Figure \ref{fig:1}-a. This fact follows immediately
by the arithmetic-geometric mean inequality, which states, for every
$a,b \geq 0$ that $(a+b)/2 \geq \sqrt{ab}$.
\begin{figure}[t]
\begin{center}
  \begin{tabular}{cc}
    \includegraphics[width=0.38\textwidth]{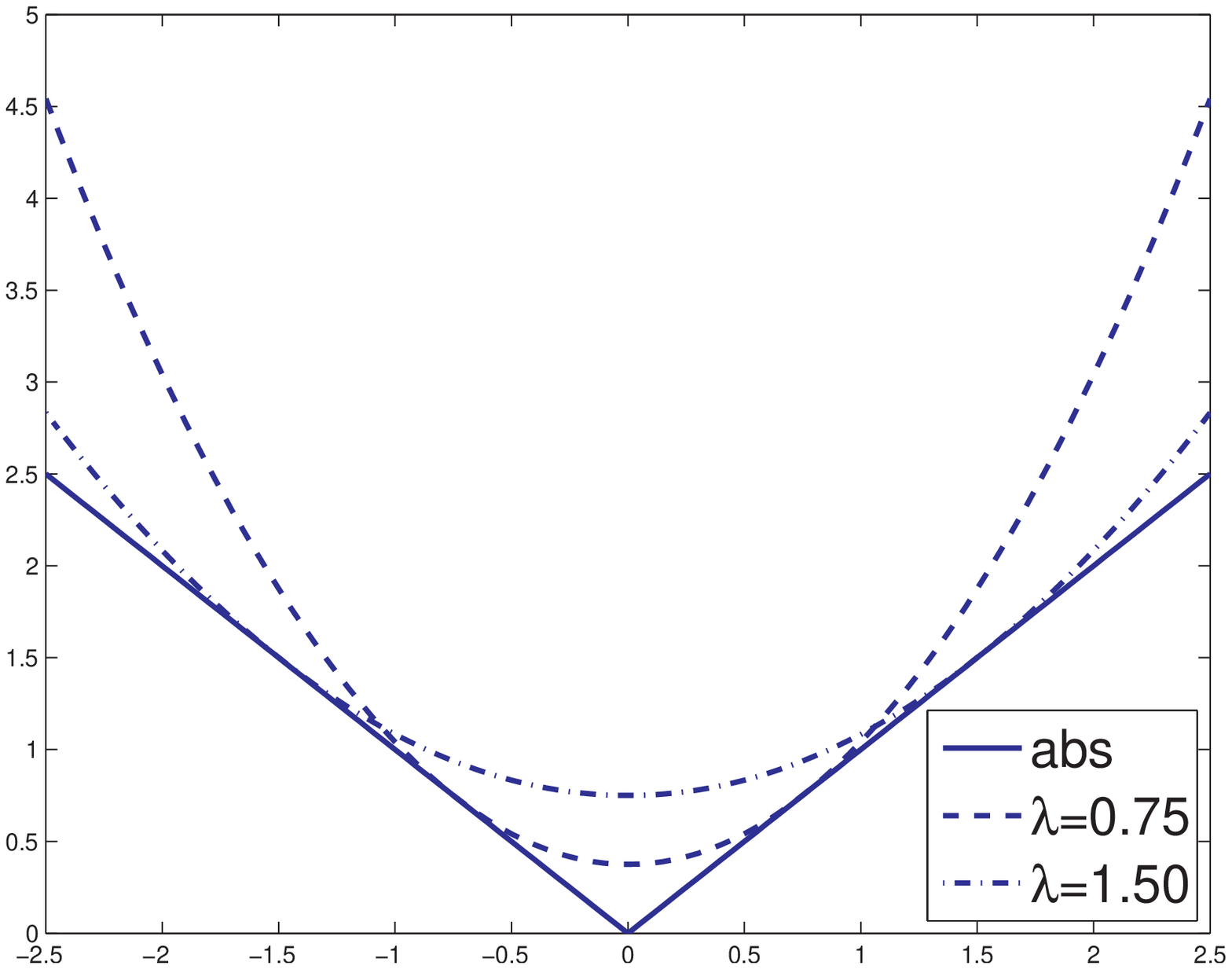} & \hspace{2.0truecm}
    \includegraphics[width=0.38\textwidth]{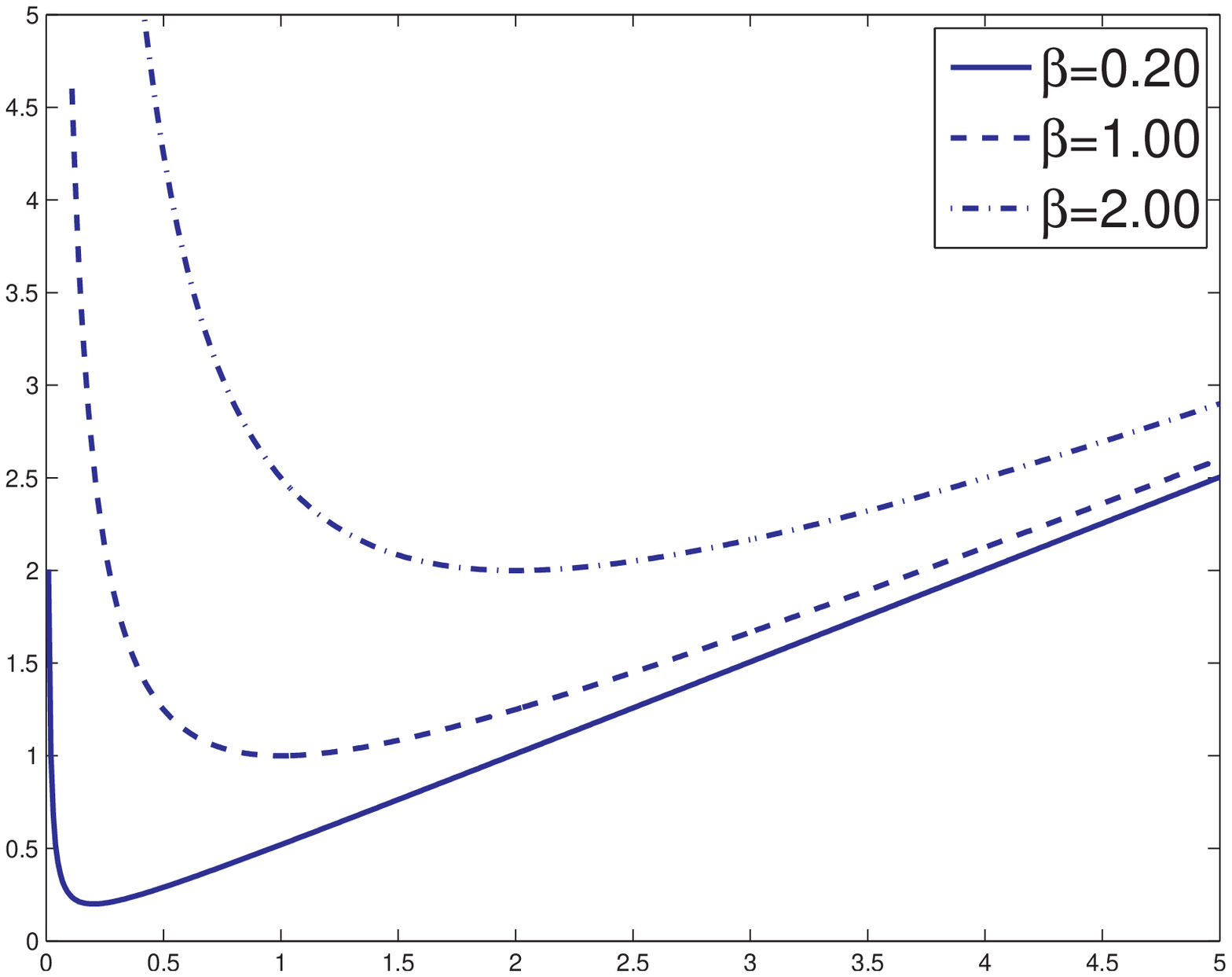} \\
    (a) & \hspace{2.0truecm}(b) \\
\end{tabular}
  \caption {(a): Function $\Gamma(\cdot,\lam)$ for some values of $\lambda>0$; (b): Function
$\Gamma(\beta,\cdot)$ for some values of $\beta \in \R$.}
  \label{fig:1}
\end{center}
\end{figure}

A special case of the formulation \eqref{eq:method} with $\La =
\R_{++}^n$ is the Lasso method \cite{lasso}, which is defined to be a solution of
the optimization problem
\beq
\min \left\{\|y - X \beta\|^2_2 + 2\rho \|\beta\|_1: \beta \in \R^n\right\}.
\nonumber \eeq
Indeed, using again the arithmetic-geometric mean inequality it follows that $\Omega(\beta|\R_{++}^n) =
\|\beta\|_1$. Moreover, if for every $i \in \N_n$ $\beta_i \neq
0$, then the infimum is attained for $\lam = |\beta|$.
This important special case motivated us to consider the general method described above.
The utility of \eqref{eq:gen-omega} is that upon inserting it into \eqref{eq:method}
there results an optimization problem over $\lam$ and $\beta$ with a continuously
differentiable objective function. Hence, we have succeeded in expressing a
nondifferentiable convex objective function by one which is continuously
differentiable on its domain.



Our first observation concerns the differentiability of $\Omega$. In
this regard, we provide a sufficient condition which ensures this
property of $\Omega$, which, although seemingly cumbersome covers
important special cases. To present our result, for any real numbers $a<b$, we define the parallelepiped $[a,b]^n = \{x: x = (x_i:i \in \N_n), a \leq x_i \leq b,~i \in \N_n\}$.
\begin{definition}
We say that the set $\La$ is admissible if it is convex and, for
all $a,b \in \R$ with $0<a<b$, the set $\La_{a,b} := [a,b]^n \cap
\La$ is a nonempty, compact subset of the interior of $\La$.
\label{def:1}
\end{definition}

\begin{proposition}
If $\beta \in (\R \backslash \{0\})^n$ and $\La$ is an admissible
subset of~$\R^n_{++}$, then the infimum above is uniquely achieved
at a point $\lam(\beta) \in \La$ and the mapping $\beta \mapsto
\lam(\beta)$ is continuous. Moreover, the function
$\Omega(\cdot|\La)$ is continuously differentiable and its partial
derivatives are given, for any $i \in \N_n$, by the formula \beq
\label{eq:part-der} \frac{\partial \Omega(\beta|\La)}{\partial
\beta_i} =  \frac{\beta_i}{\lam_i(\beta)}. \eeq \label{prop:0}
\end{proposition}
We postpone the proof of this proposition to the appendix. We note
that, since $\Omega(\cdot|\La)$ is continuous, we may compute it
at a vector $\beta$, some of whose components are zero, as a
limiting process. Moreover, at such a vector the function
$\Omega(\cdot|\La)$ is in general not differentiable, for example
consider the case $\Omega(\beta|\R^n_{++}) = \|\beta\|_1$.

The next proposition provides a justification of the penalty
function as a means to incorporate structured sparsity and
establish circumstances for which the penalty function is a norm.
To state our result, we denote by ${\bar \La}$ the closure of the
set $\La$.

\begin{proposition}
For every $\beta \in \R^n$, we have that $\|\beta\|_1 \leq
\Omega(\beta|\La)$ and the equality holds {\em if and only if}
$|\beta| := (|\beta_i|: i \in \N_n) \in \bar{\La}$.
Moreover, if $\La$ is a nonempty convex cone then the
function $\Omega(\cdot|\La)$ is a norm and we have that $\Omega(\beta|\La) \leq \omega \|\beta\|_1$,
where $\omega := \max\{\Omega(e_k|\La): k \in \N_n\}$ and $\{e_k: k \in \N_n\}$ is the canonical basis of $\R^n$.
\label{prop:ss}
\end{proposition}
\begin{proof}
By the arithmetic-geometric mean inequality we have that $\|\beta\|_1
\leq \Gamma(\beta,\lam)$, proving the first assertion. If $|\beta|
\in {\bar \La}$, there exists a sequence $\{\lam^k: k \in \N\}$ in
$\La$, such that $\lim_{k \rightarrow \infty} \lam^k = |\beta|$.
Since $\Omega(\beta|\La) \leq \Gamma(\beta,\lam^k)$ it readily
follows that $\Omega(\beta|\La) \leq \|\beta\|_1$. Conversely, if
$|\beta| \in {\bar \La}$, then there is a sequence $\{\lam^k: k
\in \N\}$ in $\La$, such that $\Gamma(\beta,\lam^k) \leq \|\beta_1\|+
1/k$. This inequality implies that some subsequence of this
sequence converges to a $\bar{\lam} \in {\bar \La}$. Using
arithmetic-geometric mean inequality we conclude that $\bar{\lam} = |\beta|$ and
the result follows. To
prove the second part, observe that if $\La$ is a nonempty convex
cone, namely, for any $\lam \in \La$ and $t \geq 0$ it holds that
$t \lam \in \La$, we have that $\Omega$ is positive homogeneous.
Indeed, making the change of variable $\lam' = \lam/|t|$ we see
that $\Omega(t \beta|\La) = |t| \Omega(\beta|\La)$. Moreover, the
above inequality, $\Omega(\beta|\La) \geq \|\beta\|_1$, implies
that if $\Omega(\beta|\La) = 0$ then $\beta=0$. The proof of the
triangle inequality follows from the homogeneity and convexity of
$\Omega$, namely $\Omega(\alpha + \beta|\La) = 2
\Omega\left((\alpha+\beta)/2|\La\right) \leq \Omega(\alpha|\La) +
\Omega(\beta|\La)$.

Finally, note that
$\Omega(\beta|\La) \leq \omega \|\beta\|_1$ if and only if $\omega =
\max\{\Omega(\beta|\La): \|\beta\|_1 = 1\}$. Since $\Omega$ is convex
the maximum above is achieved at an extreme point of the $\ell_1$ unit
ball.
\end{proof}
This proposition indicates a heuristic interpretation of the
method \eqref{eq:method}: among all vectors $\beta$ which have a fixed
value of the $\ell_1$ norm, the penalty function $\Omega$ will 
encourage those for which $|\beta| \in \La$. Moreover, when 
$|\beta| \in \La$ the function $\Omega$ reduces to the $\ell_1$ norm and, so,
the solution of problem $\eqref{eq:method}$ is expected to be 
sparse. The penalty function therefore will encourage certain desired
sparsity patterns.

The last point can be better understood by looking at problem
\eqref{eq:primal}. For every solution $(\hbeta,\lh)$, 
the sparsity pattern of $\hbeta$ is contained in the sparsity pattern
of $\lh$, that is, the indices associated with nonzero components of
$\hbeta$ are a subset of those of $\lh$. Indeed, if $\lh_i=0$ it must
hold that $\hbeta_i = 0$ as well, since the objective would diverge
otherwise (because of the ratio $\beta_i^2/\lambda_i$). Therefore, if
the set $\La$ favors certain sparse solutions of $\lh$, the same
sparsity pattern will be reflected on $\hbeta$. Moreover, the $\sum_{i
  \in \N_{n}} \lambda_i$ term appearing in the expression for
$\Gamma(\beta, \lambda)$ favors sparse $\lambda$ vectors. For
example, a constraint of the form $\lam_1\geq \dots \geq \lam_n$
favors consecutive zeros at the end of $\lambda$ and nonzeros
everywhere else. This will lead to zeros at the terminal components of
$\beta$ as well. Thus, in many cases like this, it is easy to
incorporate a convex constraint on $\lambda$, whereas it may not be
possible to do the same with $\beta$.


Next, we note that a normalized version of the group Lasso penalty
\cite{group_lasso} is included in our setting as a special case.
If, for some $k \in \N_n$, $\{J_\ell: \ell \in \N_k\}$ forms a partition of
the index set $\N_n$, the corresponding group Lasso penalty is
defined as \beq \label{eq:GL} \Omega_{\rm GL}(\beta) = \sum_{\ell
\in \N_k} \sqrt{|J_\ell|}~\|\beta_{|J_\ell}\|_2, \eeq where, for
every $J \subseteq \N_n$, we use the notation $\beta_{|J} =
(\beta_j: j \in J)$. It is an easy matter to verify that
$\Omega_{\rm GL} = \Omega(\cdot|\La)$ for $\La = \{\lam: \lam \in
\R_{++}^n, \lam_j = \theta_\ell,~j \in J_\ell,~\ell \in
\N_k,~\theta_\ell > 0\}$.

The next proposition presents a useful construction which may be
employed to generate new penalty functions from available ones. It is
obtained by composing a set $\Theta \subseteq \R_{++}^k$ with a linear transformation,
modeling the sum of the components of a vector, across the elements of a prescribed
partition ${\cal P} = \{P_\ell: \ell \in \N_k\}$ of $\N_n$. To describe our result we
introduce the {\em group average map} $A_{\cal P}:
\R^n \rightarrow \R^k$ induced by ${\cal P}$. It is defined, for each $\beta \in \R^n$, as
$A_{\cal P}(\beta) = (\|\beta_{|P_{\ell}}\|_1 : \ell \in \N_k)$.

\begin{proposition} If $\Theta \subseteq \R_{++}^k$, $\beta \in \R^n$
and ${\cal P}$ is a partition of $\N_n$ then $$ \Omega(\beta|A_{\cal P
}^{-1}(\Theta)) = \Omega(A_{\cal P}(\beta)|\Theta). $$
\label{prop:comb}
\end{proposition} \begin{proof} The idea of the proof depends on two
basic observations. The first uses the set theoretic formula $$ A_\J^{-1}(\Theta) =
\bigcup_{\theta \in \Theta} A_\J^{-1}(\theta).$$ From
this decomposition we obtain that \beq
\label{eq:charlie} \Om(\beta|A_\J^{-1}(\Theta)) = \inf \left\{\inf \left\{
\Ga(\beta,\lam): \lam \in A^{-1}_\J(\theta) \right\}
: \theta \in \Theta\right\}. \eeq Next, we write $\theta =
(\theta_\ell: \ell \in \N_k) \in \Theta$ and decompose the inner
infimum as the sum $$ \sum_{\ell \in \N_k} \inf \left\{\frac{1}{2}
\sum_{j \in J_\ell} \left(\frac{\beta_j^2}{\lam_j}+\lam_j\right):
\sum_{j \in J_\ell} \lam_j =
\theta_\ell, \lam_j > 0, j \in J_\ell \right\}.$$ Now, the second essential step in the proof
evaluates the infimum in the second sum by the Cauchy-Schwarz
inequality  to obtain that $$ \inf \left\{\Ga(\beta|\lam) : \lam \in
A^{-1}_\J(\theta) \right\}=\sum_{\ell \in \N_k}
\frac{1}{2}
\left(\frac{\|\beta_{|J_\ell}\|_1^2}{\theta_\ell}+\theta_\ell\right).$$
We now substitute this formula into the right hand side of equation
\eqref{eq:charlie} to finish the proof.\end{proof}

When the set $\La$ is a nonempty convex cone, to emphasize that
the function $\Omega(\cdot|\La)$ is a norm we denoted it by
$\|\cdot\|_\La$. We end this section with the identification of
the dual norm of $\|\cdot\|_\La$, which is defined as
$$
\|\beta\|_{*,\La} = \max \left\{ \beta\trans u : u \in \R^n,\|u\|_\La=1 \right\}.
$$
\begin{proposition}
If $\La$ is a nonempty convex cone then there holds the equation
\beq
\|\beta\|_{*,\La} = \sup\left\{\sqrt{\frac{\sum_{i\in \N_n} \lam_i \beta_i^2}{\sum_{i\in \N_n} \lam_i}}: \lam \in \La\right\}.
\nonumber
\eeq
\label{prop:dual}
\end{proposition}
\begin{proof}
By definition, $\varphi =\|\beta\|_{*,\La}$ is the smallest constant $\varphi$ such that, for every $\lam \in \La$ and $u \in \R^n$,
it holds that
$$
\frac{\varphi}{2} \sum_{i\in \N_n} \left(\frac{u_i^2}{\lam_i} + \lam_i \right)- \beta\trans u \geq 0.
$$
Minimizing the left hand side of this inequality for $u \in \R^n$ yields the equivalent inequality
$$
\varphi^2 \geq \frac{\sum_{i\in \N_n} \lam_i \beta_i^2}{\sum_{i \in \N_n} \lam_i}.
$$
Since this inequality holds for every $\lam \in \La$, the result follows by taking the supremum of the
right hand side of the above inequality over this set.
\end{proof}
The formula for the dual norm suggests that we introduce the set ${\tilde \La} = \{\lam: \lam \in \La, \sum_{i \in \N_n} \lam_i = 1\}$. With this notation we see that the dual norm becomes
$$
\|\beta\|_{*,\La} = \sup\left\{\sqrt{\sum_{i\in \N_n} \lam_i \beta_i^2}: \lam \in {\tilde \La}\right\}.
$$
Moreover, a direct computation yields an alternate form for the original norm given by the equation
$$
\|\beta\|_{\La} = \inf\left\{\sqrt{\sum_{i\in \N_n} \frac{\beta_i^2}{\lam_i}}: \lam \in {\tilde \La}\right\}.
$$



\section{Box penalty} \label{sec:box}
We proceed to discuss some examples
of the set $\La \subseteq
\R_{++}^n$ which may be used in the design of the penalty function
$\Omega(\cdot|\La)$.

The first example, which is presented in this section, corresponds to
the prior knowledge that the magnitude of the components of the
regression vector should be in some prescribed intervals.
We choose $a =(a_i:i \in \N_n)$, $b = (b_i:i \in \N_n)
\in \R^n$, $0 < a_i \leq b_i$ and define the corresponding box as $B[a,b] :=
\{(\lam_i: i \in \N_n): \lam_i \in [a_i,b_i],~i \in \N_n\}.$
The theorem below establishes the form of the box penalty.
To state our result,
we define, for every $t \in \R$, the function $t_+ = \max(0,t)$.
\begin{theorem}
\label{thm:box}
We have that
$$
\Omega(\beta|B[a,b]) = \|\beta\|_1 + \sum_{i \in \N_n} \left(\frac{1}{2a_i} (a_i-|\beta_i|)_+^2 + \frac{1}{2b_i}(|\beta_i|-b_i)_+^2\right).
$$
Moreover, the components of the vector $\lam(\beta) := {\rm argmin}\{\Gamma(\beta,\lam): \lam \in B[a,b]\}$ are given
by the equations $\lam_i(\beta) = |\beta_i|+ (a_i-|\beta_i|)_+ - (|\beta_i|-b)_+$, $i \in \N_n$.
\end{theorem}
\begin{proof}
Since $\Omega(\beta| B[a,b]) = \sum_{i\in \N_n}
\Omega(\beta_i|[a_i,b_i])$ it suffices to
establish the result in the case $n=1$. We shall show that if
$a,b,\beta \in \R$, $a\leq b$ then
\beq
\label{eq:dec11}
\Omega(\beta| [a,b])
= |\beta| + \frac{1}{2a} (a-|\beta|)_+^2 +
\frac{1}{2b}(|\beta|-b)_+^2.
\eeq
Since both sides of the above equation are continuous functions of
$\beta$ it suffices to prove this equation for $\beta \in \R
\backslash \{0\}$. In this case, the function $\Gamma(\beta,\cdot)$ is
strictly convex, and so, has a unique minimum in $\R_{++}$ at $\lam
= |\beta|$, see also Figure \ref{fig:1}-b. Moreover, if $|\beta|\leq a$ the minimum occurs at
$\lam =a$, whereas
if $|\beta|\geq b$, it occurs at $\lam=b$. This establishes the formula for
$\lam(\beta)$.
Consequently, we have that $$ \Omega(\beta|[a,b]) =
\left\{ \begin{array}{lll} \nonumber |\beta|,  & {\rm if}~ |\beta|
\in [a,b] \vspace{.2truecm} \\ \nonumber
\frac{1}{2}\left(\frac{\beta^2}{a}+ a\right),  & {\rm if}~|\beta|
< a \\ \nonumber \frac{1}{2}\left(\frac{\beta^2}{b}+ b\right),  &
{\rm if}~|\beta| > b. \end{array}\right. $$ Equation
\eqref{eq:dec11} now follows by a direct computation.
\end{proof}
We also refer to \cite{jacobT,Owen} for related penalty functions.
Note that the function in equation \eqref{eq:dec11} is a
concatenation of two quadratic functions, connected together with
a linear function. Thus, the box penalty will favor sparsity only
for $a=0$, case that is defined by a limiting argument.

\section{Wedge penalty}
\label{sec:WP}
In this section, we consider the case that the coordinates of the
vector $\lam \in \La$ are ordered in a nonincreasing fashion. As we
shall see, the corresponding penalty function favors regression vectors
which are likewise nonincreasing.

We define the wedge $$W = \{\lambda:
\lam=(\lam_i:i \in \N_n) \in \R_{++}^n, \lam_i \geq \lam_{i+1},~i
\in \N_{n-1}\}. $$ Our next result describes the form of the penalty
$\Omega$ in this case. To explain this result we require some
preparation.  We say that a partition $\J = \{J_\ell:
\ell \in \N_k\}$ of $\N_n$ is {\em contiguous} if for all
$i \in J_\ell,j \in J_{\ell+1}$, $\ell \in \N_{k-1}$,
it holds that $i <
j$. For example, if $n=3$, partitions
$\{\{1,2\},\{3\}\}$ and $\{\{1\},\{2\},\{3\}\}$ are
contiguous but $\{\{1,3\},\{2\}\}$ is not.
\begin{definition} Given any two disjoint subsets $J,K
\subseteq \N_n$ we define the region in $\R^n$ \beq
Q_{J,K}=\left\{\beta: \beta \in \R^n, \frac{\|\beta_{|J}\|_2^2}{|J|}
> \frac{\|\beta_{|K}\|_2^2}{|K|} \right\}. \eeq \end{definition}
\noindent Note that the boundary of this region is determined by
the zero set of a homogeneous polynomial of degree two. We also
need the following construction. \begin{definition}
\label{def:4.2} For every $S\subseteq\N_{n-1}$ we set $k =
|S|+1$ and label the elements of $S$ in increasing order as $S =
\{j_\ell: \ell \in \N_{k-1}\}$. We associate with the set $S$ a
contiguous partition of $\N_n$, given by $\J(S)= \{J_\ell: \ell
\in \N_k\}$, where we define $J_\ell:=[j_{\ell-1}+1,j_\ell] \cap
\N_n,$ $\ell \in \N_k$, and set $j_0=0$ and $j_k = n$.
\end{definition} Figure \ref{fig:3} illustrates an example of a
contiguous partition along with the set $\J(S)$.
\begin{figure}
\begin{center}
\includegraphics[width=0.8\textwidth]{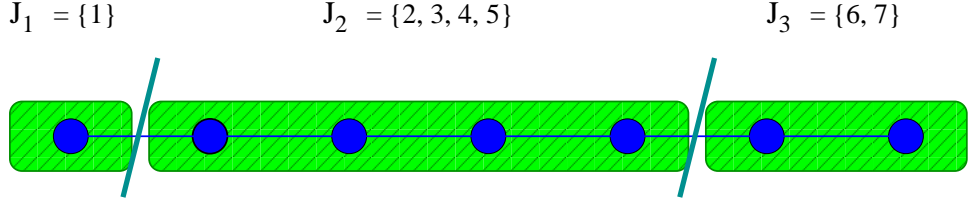}
\caption{Partition of $\beta = (1.0732, -0.4872, 0.2961, -1.3692, 1.4731, -0.0073, -0.2133)$.}
\label{fig:3}
\end{center}
\end{figure}

A subset $S$ of $\N_{n-1}$
also induces two regions in $\R^n$ which play a central role in the
identification of the wedge penalty. First, we describe the region
which ``crosses over'' the induced partition $\J(S)$. This is defined
to be the set
\beq O_S :=\bigcap
\left\{Q_{J_\ell,J_{\ell+1}}: \ell \in \N_{k-1}\right\}.
\label{eq:across} \eeq
In other words, $\beta \in O_S$ if the average of the square of its components
within each region $J_\ell$
strictly decreases with $\ell$. The next region which is essential in our analysis is the ``stays within'' region,
induced by the partition $\J(S)$. This region 
is defined as \beq I_S := \bigcap
\left\{ {\bar Q}_{J_\ell,J_{\ell,q}}: q \in J_\ell, \ell \in \N_k \right\}
\label{eq:within} \eeq
where ${\bar Q}$ denotes the closure of the set $Q$ and we use the 
notation $J_{\ell,q} := \{j: j \in J_\ell, j \leq q\}$. 
In other words, all vectors $\beta$ within this region have the
property that, for every set $J_\ell \in \J(S)$, the average of the
square of a first segment of components of $\beta$ within this set is
not greater than the average over $J_\ell$. We note
that if $S$ is the empty set the above notation should be interpreted
as $O_S = \R^n$ and $$ I_S = \bigcap \{\bar{Q}_{\N_n,\N_q}: q \in
\N_n\}.  $$

\noindent From the cross-over and stay-within sets we
define the region $$ P_S = O_S \cap I_S. $$ Alternatively, we shall describe below the
set $P_S$ in terms of two vectors induced by a vector $\beta \in
\R^n$ and the set $S \subseteq \N_{n-1}$. These vectors play the role of the
Lagrange multiplier and the minimizer $\lambda$ for the wedge
penalty in the theorem below.
\begin{definition}
\label{def:3}
For every vector $\beta \in (\R\backslash \{0\})^n$ and every subset
$S \subseteq \N_{n-1}$ we let $\J(S)$ be the induced contiguous
partition of $\N_n$ and define two vectors $\zeta(\beta,S) \in
\R_+^{n+1}$ and $\delta(\beta,S) \in \R_{++}^n$ by
\beq
\label{eq:111}
\zeta_q(\beta,S) = \left\{ \begin{array}{lll} \nonumber 0,  & {\rm if}~ q \in S \cup \{0,n\}, \\ \nonumber \\ \nonumber
|J_{\ell,q}| - |J_\ell| \frac{\|\beta_{|J_{\ell,q}}\|_2^2}{\|\beta_{|J_{\ell}}\|_2^2}
,  & {\rm if}~ q \in J_\ell, \ell \in \N_k\nonumber
\end{array}\right.
\eeq
and
\beq
\label{eq:222}
\delta_q(\beta,S) = \frac{\|\beta_{|J_\ell}\|_2}{\sqrt{|J_\ell|}},~q \in J_\ell, \ell \in \N_k.
\eeq
\end{definition}
\noindent Note that the components of $\delta(\beta,S)$ are constant on
each set $J_\ell$, $\ell \in \N_k$.

\begin{lemma}
\label{lem:small}
For every $\beta \in (\R \backslash \{0\})^n$ and $S \subseteq \N_{k-1}$ we have that
\begin{enumerate}
\item[(a)]
$\beta \in P_S$ {\em if and only if} $\zeta(\beta,S) \geq 0$ and $\delta(\beta,S) \in {\rm int}(W)$;
\item[(b)] If $\delta(\beta,S_1) = \delta(\beta,S_2)$ and $\beta \in O_{S_1} \cap O_{S_2}$ then $S_1 = S_2$.
\end{enumerate}
\end{lemma}
\begin{proof}
The first assertion follows directly from the definition of the
requisite quantities. The proof of the second assertion is a direct
consequence of the fact that the vector $\delta(\beta,S)$ is a constant
on any element of the partition $\J(S)$ and strictly decreasing from
one element to the next in that partition.
\end{proof}

For the theorem below we introduce, for every $S \in \N_{n-1}$ the sets
$$
U_S := P_S \cap (\R \backslash \{0\})^n.
$$
We shall establishes not only that the collection of sets $\U := \{U_S: S \subseteq
\N_{n-1}\}$ form a {\em partition} of $(\R \backslash \{0\})^n$, that is, their union is
$(\R \backslash \{0\})^n$ and two distinct elements of
$\U$ are disjoint, but also explicitly determine the wedge
penalty on each element of $\U$.

\begin{theorem} \label{thm:sp} The collection of sets $\U := \{U_S: S \subseteq \N_{n-1}\}$ form a partition of $(\R\backslash \{0\})^n$. For each $\beta \in (\R\backslash \{0\})^n$ there is a unique
$S \subseteq \N_{n-1}$ such that $\beta \in {\cal U}_S$, and
\beq \|\beta\|_{W} = \sum_{\ell \in
\N_{k}} \sqrt{|J_\ell|} \|\beta_{|J_\ell}\|_2,
\label{eq:omline}
\eeq where $k=|S|+1$.
Moreover, the components of the vector $\lam(\beta) := {\rm argmin}\{\Gamma(\beta,\lam): \lam \in W\}$ are given
by the equations $\lam_j(\beta) = \mu_\ell,~j \in J_\ell,~\ell \in \N_k$, where
\beq \mu_\ell =
\frac{\|\beta_{|J_\ell}\|_2}{\sqrt{|J_\ell|}}.
\label{eq:minmu}
\eeq
\end{theorem}

\begin{proof} First, let us observe that there are $n-1$ inequality
constraints defining $W$. It readily follows that all vectors
in this constraint set are {\it regular}, in the sense of
optimization theory, see \cite[p. 279]{Bert}. Hence, we can appeal to
\cite[Prop. 3.3.4,~p. 316 and Prop. 3.3.6,~p. 322]{Bert}, which state that $\lam \in \R_{++}^n$ is a solution to the minimum problem determined
by the wedge penalty, if and only if there exists a vector
$\alpha=(\alpha_i: i \in \N_{n-1})$ with nonnegative components such
that \beq -\frac{\beta^2_j}{\lam_j^2} + 1 + \alpha_{j-1}-\alpha_j =
0,~~~j \in \N_n \label{eq:1st-hh}, \eeq where we set $\alpha_0 =
\alpha_n = 0.$ Furthermore, the following complementary slackness
conditions hold true \beq
\label{eq:2nd-hh}
\alpha_{j} (\lam_{j+1} -
\lam_j) = 0,~j \in \N_{n-1}.
\eeq
To unravel these equations,
we let ${\hat S}:= \{j: \lam_{j} > \lam_{j+1}, j \in \N_{n-1}\}$, which is
the subset of indexes corresponding to the constraints that are not tight.
When $k \geq 2$, we express this set in the form $\{j_\ell: \ell \in \N_{k-1}\}$
where $k = |{\hat S}|+1$.
As explained in Definition \ref{def:4.2}, the set $\hS$ induces the partition $\J(\hS) = \{J_\ell: \ell \in \N_k\}$ of
$\N_n$.
When $k=1$ our notation should be interpreted to mean that $\hS$ is empty
and the partition $\J(\hS)$ consists only of $\N_n$. In this case, it is easy to
solve equations \eqref{eq:1st-hh} and \eqref{eq:2nd-hh}. In fact,
all components of the vector $\lam$ have a common value, say $\mu > 0$, and by summing both sides of equation \eqref{eq:1st-hh} over $j \in \N_n$ we obtain that
$$
\mu^2 = \frac{\|\beta\|_2^2}{n}.
$$
Moreover, summing both sides of the same equation over $j \in \N_q$ we obtain
that
$$\alpha_q = - \frac{\sum_{j \in \N_q} \beta_j^2}{\mu^2} + q
$$ and,
since $\alpha_q \geq 0$ we conclude that $\beta \in I_\hS=P_\hS$.

We now consider the case that $k \geq 2$. Hence, the vector $\lam$ has
equal components on each subset $J_\ell$, which we denote by
$\mu_\ell$, $\ell \in \N_{k-1}$.  The definition of the set $\hS$
implies that the sequence $\{\mu_\ell: \ell \in \N_k\}$ is strictly decreasing and equation
\eqref{eq:2nd-hh} implies that $\alpha_j = 0$, for every $j \in \hS$.
Summing both sides of equation \eqref{eq:1st-hh} over $j \in J_\ell$
we obtain that \beq
\label{eq:aux} -\frac{1}{\mu_\ell^2} \sum_{j \in J_\ell} \beta_j^2 +
|J_\ell| = 0 \eeq from which equation \eqref{eq:minmu} follows.
Since the $\mu_\ell$ are strictly decreasing, we conclude that $\beta \in O_\hS$.
Moreover, choosing $q \in J_\ell$ and
summing both sides of equations \eqref{eq:1st-hh} over $j \in J_{\ell,q}$
we obtain that
$$
0 \leq \alpha_q = -\frac{\|\beta_{|J_{\ell,q}}\|_2^2}
{\mu_\ell^2} +|J_{\ell,q}|
$$
which implies that $\beta \in {\bar Q}_{J_\ell,J_{\ell,q}}$. Since this holds for
every $q \in J_\ell$ and $\ell \in N_k$ we conclude that $\beta \in I_\hS$ and therefore, it follows that $\beta \in U_S$.

In summary, we have shown that $\alpha = \zeta(\beta,\hS)$, $\lam =
\delta(\beta,\hS)$, and $\beta \in U_\hS$. In particular, this implies
that the collection of sets $\U$ covers $\zRn$. Next, we show that the
elements of ${\cal U}$ are disjoint. To this end, we observe that, the
computation described above can be {\em reversed}. That is to say,
conversely for {\em any} $\hS
\subseteq \N_{n-1}$ and $\beta \in U_\hS$
we conclude that $\delta(\beta,\hS)$ and $\zeta(\beta,\hS)$ solve the
equations \eqref{eq:1st-hh} and \eqref{eq:2nd-hh}. Since the wedge
penalty function is {\em strictly convex} we know that equations
\eqref{eq:1st-hh} and \eqref{eq:2nd-hh} have a unique solution. Now, if
$\beta \in U_{S_1} \cap U_{S_2}$ then it must follow that
$\delta(\beta,S_1) =
\delta(\beta,S_2)$. Consequently, by part (b) in Lemma \ref{lem:small}
we conclude that $S_1=S_2$.
\end{proof}
Note that the set $S$ and the associated partition $\J$ appearing
in the theorem is identified by examining the optimality conditions of the optimization 
problem \eqref{eq:gen-omega} for $\La=W$. There are $2^{n-1}$ possible partitions. Thus, for a given $\beta \in \zRn$, determining the corresponding partition is a challenging problem.  We explain how to do
this in Section \ref{sec:algo}.

An interesting property of the Wedge penalty, which is indicated by
Theorem \ref{thm:sp}, is that it has the form of a group Lasso penalty
as in equation \eqref{eq:GL}, with groups not fixed {\em a-priori} but
depending on the location of the vector $\beta$. The groups are the
elements of the partition $\J$ and are identified by certain convex
constraints on the vector $\beta$. For example, for $n=2$ we obtain
that $\Omega(\beta|W) =
\|\beta\|_1$ if $|\beta_1| > |\beta_2|$ and $\Omega(\beta|W) =
\sqrt{2}\|\beta\|_2$ otherwise.
For $n=3$, we have that
$$
\Omega(\beta|W) = \left\{ \begin{array}{lll} \nonumber
\|\beta\|_1,  & {\rm if~} |\beta_1| >
|\beta_2| > |\beta_3| & ~~\J = \{\{1\},\{2\},\{3\}\}
 \\ \\
\nonumber \sqrt{2(\beta_1^2 +
\beta_2^2)}+|\beta_3|,  & {\rm if~} |\beta_1| \leq |\beta_2| ~~{\rm and~~}
\frac{\beta_1^2+\beta_2^2}{2} > \beta_3^2 & ~~\J = \{\{1,2\},\{3\}\}
\\ \\ \nonumber
\nonumber
 |\beta_1|+ \sqrt{2(\beta_2^2 +
\beta_3^2)}, & {\rm if~} |\beta_2| \leq
|\beta_3|~~{\rm and~~}\beta_1^2 >
\frac{\beta_2^2+\beta_3^2}{2} & ~~\J = \{\{1\},\{2,3\}\}
\\ \\\nonumber \nonumber
 \sqrt{3(\beta_1^2 + \beta_2^2 +\beta_3^2)},  &
 {\rm otherwise} & ~~\J = \{\{1,2,3\}\}
\nonumber \end{array}\right. $$
\begin{figure}[!t]
\begin{center}
\begin{tabular}{ccccc}
\includegraphics[width=0.168\textwidth]{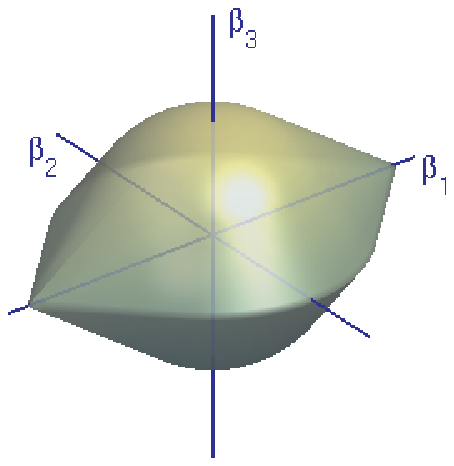} &
\includegraphics[width=0.168\textwidth]{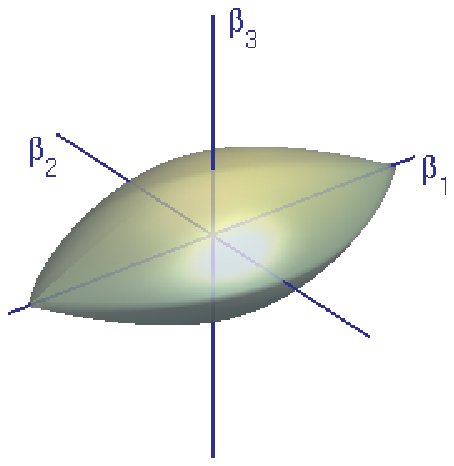} &
\includegraphics[width=0.168\textwidth]{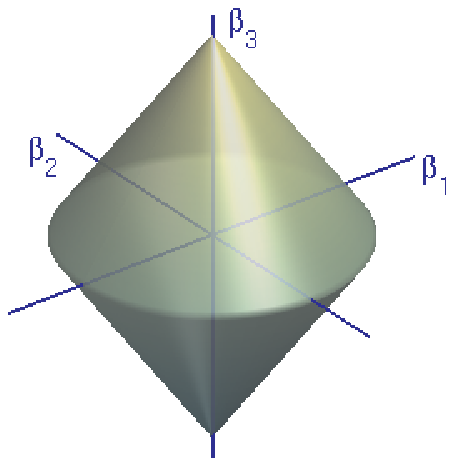} &
\includegraphics[width=0.168\textwidth]{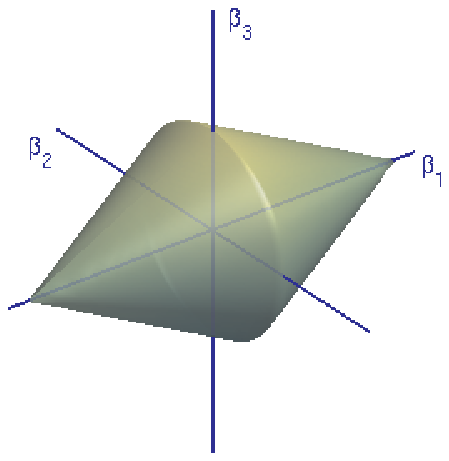} &
\includegraphics[width=0.168\textwidth]{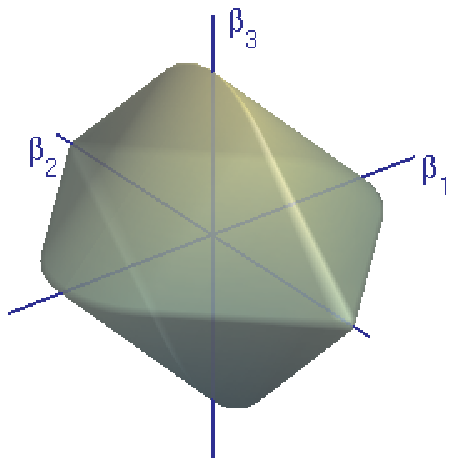} \\
(a) & (b) & (c) & (d) & (e)\\
\end{tabular}
\caption{Unit ball of different penalty functions: (a) Wedge penalty $\Omega(\cdot|W)$; (b) hierarchical group Lasso;
(c) group Lasso with groups $\{\{1,2\},\{3\}\}$; (d) group Lasso with groups $\{\{1\},\{2,3\}\}$;
(e) the penalty $\Omega(\cdot|W^2)$.}
\label{fig:4}
\end{center}
\end{figure}

\noindent where we have also displayed the partition $\J$ involved
in each case. We also present a graphical representation of the
corresponding unit ball in Figure \ref{fig:4}-a. For comparison we
also graphically display the unit ball for the hierarchical group
Lasso with groups $\{1,2,3\},\{2,3\},\{3\}$ and two group Lasso in
Figure \ref{fig:4}-b,c,d, respectively.

The wedge may equivalently be expressed as the constraint that the
difference vector $D^{1}(\lam):=(\lam_{j+1}-\lam_{j}: j \in \N_{n-1})$
is less than or equal to zero. This alternative interpretation suggests the $k$-th order difference operator,
which is given by the formula
$$
D^{k}(\lam) =
\left(\lam_{j+k}+\sum_{\ell \in \N_k} (-1)^\ell \binom{k}{\ell} \lam_{j+k-\ell}: j \in
\N_{n-k}\right)
$$
and the corresponding $k$-th wedge
\beq
W^{k} := \{\lam: \lam \in
\R^n_{++},~D^{k}(\lam) \geq 0\}.
\label{eq:k-wedge}
\eeq
The associated penalty $\Omega(\cdot|W^{k})$ encourages vectors
whose sparsity pattern is concentrated on at most $k$ different
contiguous regions. Note that $W^{1}$ is not the wedge $W$ considered earlier.
Moreover, the $2$-wedge includes vectors which have a convex ``profile'' and whose
sparsity pattern is concentrated either on the first elements of the vector, on the last, or on both.

\section{Graph penalty}
\label{sec:GP}
In this section we present an extension of the wedge set which is
inspired by previous work on the group Lasso estimator with
hierarchically overlapping groups
\cite{binyu}. It models vectors whose magnitude is ordered according to
a graphical structure.

Let $G=(V,E)$ be a directed graph, where $V$ is the set of $n$
vertices in the graph and $E \subseteq V \times V$ is the edge set,
whose cardinality is denoted by $m$. If $(v,w) \in E$ we say that
there is a directed edge from vertex $v$ to vertex $w$. The graph is
identified by the $m \times n$ {\em incidence matrix}, which we define
as $$A_{e, v} = \left\{ \begin{array}{rll} \nonumber 1, & {\rm if}~
e=(v,w) \in E,~w \in V,
\vspace{.2truecm} \\ \nonumber
-1,  &
{\rm if}~e=(w,v) \in E,~ w \in V, \nonumber \vspace{.2truecm}
\\
0, & {\rm otherwise}.
\end{array}\right.
$$ We consider the penalty $\|\cdot\|_{\La_G}$ for the convex cone $\La_G = \{\lam:
\lam \in \R^n_{++}, A\lam \geq 0\}$ and assume, from now on, that $G$ is
acyclic (DAG), that is, $G$ has no directed loops.  In particular, this
implies that, if $(v,w) \in E$ then $(w,v) \notin E$.  The wedge
penalty described above is a special case of the graph penalty corresponding to a line graph.
Let us now discuss some aspects of the graph penalty for an arbitrary DAG. As we shall see, our remarks lead
to an explicit form of the graph penalty when $G$ is a tree.

If $(v,w) \in E$ we say that vertex $w$ is a child of vertex $v$
and $v$ is a parent of $w$. For every vertex $v \in V$, we let
$C(v)$ and $P(v)$ be the set of children and parents of $v$,
respectively. When $G$ is a tree, $P(v)$ is the empty set if $v$
is the root node and otherwise $P(v)$ consists of only one
element, the parent of $v$, which we denote by $p(v)$.

Let $D(v)$ be the set of descendants of $v$, that is, the set of
vertices which are connected to $v$ by a directed path starting in $v$,
and let $A(v)$ be the set of ancestors of $v$, that is, the set of
vertices from which a directed path leads to $v$. We use the
convention that $v \in D(v)$ and $v \notin A(v)$.

Every connected subset $V' \subseteq V$ induces a subgraph of $G$
which is also a DAG.
If $V_1$ and $V_2$ are
disjoint connected subsets of $V$, we say that they are connected if there is
at least one edge connecting a pair of vertices in $V_1$ and $V_2$, in either one or the other direction.
Moreover, we say that $V_2$ is below $V_1$ --- written $V_2 \Downarrow V_1$ --- if $V_1$
and $V_2$ are connected and every edge connecting them departs from
a node of $V_1$.

\begin{definition}
Let $G$ be a DAG. We say that $C
\subseteq E$ is a cut of $G$ if it induces a partition ${\cal V}(C) =
\{V_\ell: \ell \in
\N_k\}$ of the vertex set $V$ such that $(v,w) \in C$ if and only if
vertices $v$ and $w$ belong to two different elements of the
partition.
\end{definition}
In other words, a cut separates a connected graph in two or more
connected components such that every pair of vertices
corresponding to a disconnected edge, that is an element of $C$,
are in two different components.  We also denote by ${\cal C}(G)$
the set of cuts of $G$, and by $D_\ell(v)$ the set of descendants
of $v$ within set $V_\ell$, for every $v \in V_\ell$ and $\ell \in
\mathbb{N}_k$. 

Next, for every $C \in {\cal C}(G)$, we define the regions in $\R^n$
by the equations
\beq O_C
= \bigcap \left\{Q_{V_1,V_2}:~ V_1,V_2 \in {\cal V}(C), V_2 \Downarrow V_1 \right\}
\label{eq:across-dag} \eeq and \beq
I_C = \bigcap \left\{{\bar Q}_{D_\ell(v), V_\ell}: \ell \in \N_k, v \in V_\ell\right\}.
\label{eq:within-dag}
\eeq
These sets are the graph equivalent of the sets defined by equations
\eqref{eq:across} and \eqref{eq:within} in the special case of the
wedge penalty in Section \ref{sec:WP}. We also set $P_C = O_C \cap
I_C$.

Moreover, for every $C \in {\cal C}(G)$, we define the sets
$$
U_C := P_C \bigcap (\R \backslash \{0\})^n.
$$
As of yet, we cannot extend Theorem \ref{thm:sp} to the case of an
arbitrary DAG. However, we can accomplish this when $G$ is a tree.

\begin{lemma}
Let $G=(V,E)$ be a tree, let $A$ be the associated incidence matrix and
let $z=(z_v: v \in V) \in \R^n$. The following facts are
equivalent:
\begin{enumerate}
\item[(a)] For every $v \in V$ it holds that
$$
\sum_{u \in D(v)}  z_u \geq 0.
$$
\item[(b)] The linear system $A\trans \alpha = -z$ admits a non-negative solution for $\alpha = (\alpha_e: e \in E)\in \R^{m}$.
\end{enumerate}
\label{lem:jean}
\end{lemma}
\begin{proof}
The incident matrix of a tree has the property that, for every $v \in
V$ and $e \in E$,
\beq
\label{eq:h1h}
\sum_{u \in D(v)} A_{eu} = -\delta_{e,(p(v),v)}
\eeq
where, for every $e,e' \in E$, $\delta_{e,e'} = 1$ if $e=e'$ and zero otherwise.
The linear system in (b) can be written componentwise as
$$
\sum_{e \in E} A_{eu} \alpha_e = -z_u.
$$
Summing both sides of this equation over $u \in D(v)$ and using equation \eqref{eq:h1h}, we obtain the equivalent equations
$$
\alpha_{(p(v),v)} = \sum_{u \in D(v)} z_u.
$$
The result follows.
\end{proof}

\begin{definition}
\label{def:3t} Let $G=(V,E)$ be a DAG. For every vector $\beta \in
(\R\backslash \{0\})^n$ and every cut $C \in {\cal C}(G)$ we let
${\cal V}(C)=\{V_\ell: \ell \in \N_k\}$, $k \in \N_n$ be the
partition of~$V$ induced by $C$, and define two vectors
$\zeta(\beta,C) \in \R_+^{n-1}$ and $\delta(\beta,C) \in
\R_{++}^n$. The components of $\zeta(\beta,C)$ are given as \beq
\label{eq:111t} \zeta_e(\beta,C) = \left\{ \begin{array}{lll}
\nonumber 0,  & {\rm if}~ e \in C, \\ \nonumber \\ \nonumber
|V_\ell|
\frac{\|\beta_{|D_\ell(u)}\|_2^2}{\|\beta_{|V_{\ell}}\|_2^2}-|D_\ell(u)|
,  & {\rm if}~ e = (u, v), u \in V_\ell, v \in D_\ell(u),~\ell \in
\N_k\nonumber
\end{array}\right.
\eeq
whereas the components of $\delta(\beta,C)$ are given by
\beq
\label{eq:222t}
\delta_v(\beta,C) = \frac{\|\beta_{|V_\ell}\|_2}{\sqrt{|V_\ell|}},~v \in V_\ell,~\ell \in \N_k.
\eeq
\end{definition}
Note that the notation we adopt in this definition differs from that used in the case of line graph, given in Definition
\ref{def:3}. However, Definition \ref{def:3t} leads to a more appropriate presentation of our results for a tree.

\begin{proposition}
Let $G=(V,E)$ be a tree and $A$ the associated incidence matrix.
For every $\beta \in (\R\backslash\{0\})^n$ and every cut $C \in {\cal C}(G)$ we have that
\begin{enumerate}
\item[(a)]
$\beta \in P_C$ {\em if and only if} $\zeta(\beta,C) \geq 0$, $A \delta(\beta,C) \geq 0$ and
$\delta_v(\beta,C) > \delta_w(\beta,C)$, for all $v \in V_1, w \in V_2$, $(v,w) \in E$, $V_1,V_2 \in {\cal V}(C)$;
\item[(b)] If $\delta(\beta,C_1) = \delta(\beta,C_2)$ and $\beta \in O_{C_1} \cap O_{C_2}$ then $C_1 = C_2$.
\end{enumerate}
\label{prop:tree}
\end{proposition}
\begin{proof}
We immediately see that $\beta \in O_C$ if and only if $A
\delta(\beta,C) \geq 0$ and $\delta_v(\beta,C) > \delta_w(\beta,C)$
for all $v \in V_1, w \in V_2$, $(v,w) \in E$, $V_1,V_2 \in {\cal
V}(C)$. Moreover, by applying Lemma \ref{lem:jean} on each element
$V_\ell$ of the partition induced by $C$ and choosing $z =
(|V_\ell|\frac{\beta^2_v}{\|\beta_{|V_\ell}\|_2^2} -1 : v \in
V_\ell)$, we conclude that $\zeta(\beta,C) \geq 0$ if and only if
$\beta \in I_C$. This proves the first assertion.

The proof of the second assertion is a direct consequence of the fact that the vector $\delta(\beta,C)$ is a constant
on any element of the partition ${\cal V}(C)$ and strictly decreasing from
one element to the next in that partition.
\end{proof}

\begin{theorem} \label{thm:tree} Let $G=(V,E)$ be a tree. The collection of sets ${\cal U} := \{U_C: C \in {\cal C}(G)\}$ form a partition of $(\R \backslash \{0\})^n$.
Moreover, for every $\beta \in (\R \backslash \{0\})^n$ there is a
unique $C \in {\cal C}(G)$ such that \beq \|\beta\|_{\La_G} =
\sum_{V_\ell \in \cal{V}(C)} \sqrt{|V_\ell|} \|\beta_{|V_\ell}\|_2
\eeq
and the vector $\lam(\beta)=(\lam_v(\beta): v \in V)$ has components
given by $\lam_v(\beta) = \mu_\ell,~v \in V_\ell$, $\ell \in \N_k$,
where \beq
\mu_\ell = \sqrt{\frac{1}{n_\ell}\sum\limits_{w \in V_\ell}
\beta_w^2}. \label{eq:minmu-tree} \eeq \end{theorem}

\begin{proof}
The proof of this theorem proceeds in a fashion similar to that of Theorem \ref{thm:sp}.
In this regard, Lemma \ref{lem:jean} is crucial.
By KKT theory (see e.g. \cite[Theorems 3.3.4,3.3.7]{Bert}),
$\lambda$ is an optimal solution of the graph penalty if and only if
there exists $\alpha \geq 0$ such that, for every $v \in V$
$$
-\frac{\beta^2_v}{\lam_v^2} + 1 -
\sum_{e \in E} \alpha_e A_{ev} = 0
$$
and the following complementary conditions hold true
\beq
\alpha_{(v,w)} (\lam_w - \lam_v)
= 0,~v \in V, w \in C(v).
\label{eq:comp-tree}
\eeq
We rewrite the first equation as
\beq
\alpha_{(p(v),v)} -\sum_{w \in C(v)} \alpha_{(v,w)} = \frac{\beta^2_v}{\lam_v^2} -1.
\label{eq:1st-bis}
\eeq
Now, if $\lam \in \La_G$ solves equations \eqref{eq:comp-tree} and
\eqref{eq:1st-bis}, then it induces a cut $C \subset E$ and a
corresponding partition ${\cal V}(C) =\{V_\ell: \ell \in \N_k\}$ of
$V$ such that $\lam_v= \mu_\ell$ for every $v \in V_\ell$. That is,
$\lam_v = \lam_w$ for every $v,w \in V_\ell$, $\ell \in \N_k$, and
$\alpha_{e} = 0$ for every $e \in C$. Therefore, summing equations
\eqref{eq:1st-bis} for $v \in V_\ell$ we get that $$
\mu_\ell = \frac{\|\beta_{|V_\ell}\|_2}{\sqrt{|V_\ell|}}.
$$
Moreover, since $\mu_\ell > \mu_q$, if $V_q \Downarrow V_\ell$ we see that $\beta \in O_C$.
Next, for every $\ell \in \N_k$ and $u \in V_\ell$ we sum both sides of equation \eqref{eq:1st-bis} for $v \in D_\ell(u)$ to obtain that
\beq
\alpha_{(p(u),u)} = \frac{\|\beta_{|D_\ell(u)}\|_2^2}{\mu_\ell^2} - |D_\ell(u)|.
\label{eq:olo}
\eeq
We see that $\beta \in I_C$ and conclude that $\beta \in U_C$.

In summary we have shown that the collection of sets ${\cal U}$ cover
$(\R \backslash \{0\})^n$.  Next, we show that the elements of ${\cal
U}$ are disjoint. To this end, we observe that, the computation
described above can be {\em reversed}. That is to say, conversely for
{\em any} partition $C=\{V_i : i \in \N_k\}$ of $V$ and $\beta \in U_C$
we conclude by Proposition \ref{prop:tree} that the vectors
$\delta(\beta,C)$ and $\zeta(\beta,C)$ solves the KKT optimality
conditions. Since this solution is unique if $\beta \in U_{C_1} \cap
U_{C_2}$ then it must follow that $\delta(\beta,C_1) =
\delta(\beta,C_2)$, which implies that $C_1=C_2$.
\end{proof}

Theorems \ref{thm:sp} and \ref{thm:tree} fall into the category of a
set $\La \subseteq \R^n$ chosen in the form of a polyhedral cone, that
is $$\La = \{\lam: \lam \in \R^n, A \lam \geq 0\} $$ where $A$ is an
$m \times n $ matrix. Furthermore, in the line graph of Theorem
\ref{thm:sp} and also the extension in Theorem \ref{thm:tree} the
matrix $A$ only has elements which are $-1,1$ or $0$. These two
examples that we considered led to explicit description of the norm
$\|\cdot\|_\La$. However, there are seemingly simple cases of a matrix
$A$ of this type where the explicit computation of the norm $\|\cdot\|_\La$
seem formidable, if not impossible. For example, if $m=2$, $n = 4$ and
$$ A = \qquad
\begin{bmatrix}
-1 & -1 & 1 & 0 \\
0 & - 1& -1 & 1
\end{bmatrix}
$$
we are led by KKT to a system of equations that, in the case of two active constraints, that is, $A\lam = 0$, are the common zeros of two {\em fourth order} polynomials in the vector $\lam \in \R^2$.

\section{Duality}
\label{sec:comp}
In this section, we comment on the utility of the class of penalty
functions considered in this paper, which is fundamentally based on
their construction as constrained infimum of quadratic functions. To
emphasize this point both theoretically and computationally, we
discuss the conversion of the regularization variational problem
over $\beta \in \R^n$, namely
\beq {\cal E}(\La) =
\inf\left\{E(\beta,\lam): \beta \in
\R^n, \lam \in \La \right\} \label{eq:ddd} \eeq
where
$$
E(\beta,\lam) := \|y-X\beta\|^2_2 + 2 \rho \Gamma(\beta,\lam),
$$
into a variational problem over $\lam \in \La$.

To explain what we have in mind, we introduce the following definition.
\begin{definition}
For every $\lam \in \R_+^n$, we define the vector $\beta(\lam) \in \R^n$ as
$$
\beta(\lam) = {\rm diag}(\lam) M(\lam)
X\trans y
$$
where $M(\lam):=({\rm diag}(\lam) X\trans X+ \rho I)^{-1}$.
\end{definition}
Note that $\beta(\lam) = {\rm argmin} \{ E(\beta,\lam): \beta \in \R^n\}$.

\begin{theorem} For $\rho >
0$, $y \in \R^m$, any $m \times n$ matrix $X$ and any nonempty convex set $\La$ we have that
\beq
{\cal E}(\La) = \min \left\{ \rho y\trans\left(X {\rm
 diag}(\lam) X\trans + \rho I\right)^{-1}y+ \rho {\rm tr} ({\rm diag}(\lam)) : \lam \in {\bar \La} \cap \R_{+}^n\right\}
\label{eq:al-reg}
\eeq
Moreover, if $\hlam$ is a solution to this problem, then $\beta(\hlam)$ is a solution to problem \eqref{eq:ddd}.
\label{thm:interchange}\end{theorem}
\begin{proof}
We substitute the formula for $\Omega(\beta|\La)$ into the right hand side of equation
\eqref{eq:ddd} to obtain that
\beq
\label{eq:pinf}
{\cal E}(\La) = \inf \left\{H(\lam): \lam \in \La
\right\}
\eeq
where we define
$$
H(\lam) = \min \left\{E(\beta,\lam)
: \beta
\in \R^n\right\}.
$$
A straightforward computation yields that
$$
H(\lam) = \rho y\trans\left(X {\rm
 diag}(\lam) X\trans + \rho I\right)^{-1}y+ \rho {\rm tr} ({\rm diag}(\lam)).
$$
Since $H(\lam)\geq \rho {\rm tr}({\rm diag}(\lam))$, we conclude that any minimizing sequence for the optimization
problem on the right hand side of equation \eqref{eq:pinf} must have a subsequence which converges. These remarks confirm
equation \eqref{eq:al-reg}.

We now prove the second claim. For $\lam \in \R_{++}^n$ a direct computation confirms that
$$
\Gamma(\beta(\lam),\lam) = \frac{1}{2} \left( y\trans X M(\lam) {\rm diag}(\lam) M(\lam) X\trans y + \trace(
{\rm diag}(\lam))  \right).
$$
Note that the right hand side of this equation provides a continuous extension of the left hand side to $\lam \in \R_+^n$.
For notational simplicity, we still use the left hand side to denote this {\em continuous extension}.

By a limiting argument, we conclude, for every $\lam \in {\bar \La}$, that
\beq
\label{eq:uy}\Omega(\beta(\lam)|\La) \leq \Gamma(\beta(\lam),\lam).
\eeq
We are now ready to complete the proof of the theorem. Let $\hlam$ be a solution for the optimization
problem \eqref{eq:al-reg}. By definition, it holds, for any $\beta \in \R^n$ and $\lam \in {\bar \La}$, that
$$
\|y-X\beta(\hlam)\|^2_2 + 2 \rho \Gamma(\beta(\hlam),\hlam) =H(\hlam) \leq H(\lam) \leq \|y-X\beta\|^2_2 + 2 \rho
\Gamma(\beta,\lam).
$$
Combining this inequality with inequality \eqref{eq:uy} evaluated at $\lam = \hlam$, we conclude that
$$\|y-X\beta(\hlam)\|^2_2 + 2 \rho \Omega(\beta(\hlam)|\La) \leq \|y-X\beta\|^2_2 + 2 \rho
\Gamma(\beta,\lam)
$$
from which the result follows.
\end{proof}
An important consequence of the above theorem is a method to find a solution $\hbeta$ to
the optimization problem \eqref{eq:ddd} from a solution to the optimization problem \eqref{eq:al-reg}.
We illustrate this idea in the case that $X = I$.
\begin{corollary}
It holds that
\beq
\min \left\{ \|\beta-y\|_2^2 + 2\rho \Omega(\beta|\La) : \beta \in \R^n\right\} = \rho \min\left\{ \sum_{i \in \N_n} \frac{y_i^2}{\lam_i + \rho} +\lam_i: \lam \in {\bar \La}\right\}.
\label{eq:corI}
\eeq
Moreover, if $\hlam$ is a solution of the right optimization problem
then the vector $\beta(\hlam) = (\beta_i(\hlam): i \in \N_n)$, whose
components are defined for $i \in \N_n$ as
\beq
{\beta}_i(\hlam) = \frac{\hlam_i y_i}{\hlam_i+ \rho}
\eeq
is a solution of the left optimization problem problem.
\label{cor:I}
\end{corollary}
We further discuss two choices of the set $\La$ in which we are able to solve problem \eqref{eq:corI} analytically.
The first case we consider is $\La = \R_{++}^n$, which corresponds to the Lasso penalty.
It is an easy matter to see that $\hlam = (|y|-\rho)_+$ and the corresponding
regression vector is obtained by the well-known ``soft thresolding'' formula $\beta(\hlam)=(|y|-\rho)_+ {\rm sign}(y)$.
The second case is the Wedge penalty. We find that the solution of the optimization problem in the right
hand side of equation \eqref{eq:corI} is $\hlam = (\lam(y)-\rho)_+$, where $\lam(y)$ is given
in Theorem \ref{thm:sp}.
Finally, we note that Corollary \ref{cor:I} and the example following it extend to
the case that $X\trans X= I$ by replacing throughout the vector $y$ by the vector $X\trans y$.
In the statistical literature this setting is referred to as orthogonal design.

\section{Optimization method}
\label{sec:algo}
In this section, we address the issue of implementing the learning
method \eqref{eq:method} numerically.

Since the penalty function
$\Omega(\cdot|\La)$ is constructed as the infimum of a family of
quadratic regularizers, the optimization problem \eqref{eq:method}
reduces to a simultaneous minimization over the vectors $\beta$ and
$\lam$. For a fixed $\lam \in \La$, the minimum over $\beta \in \R^n$
is a standard Tikhonov regularization and can be solved directly in
terms of a matrix inversion. For a fixed $\beta$, the minimization over
$\lam \in \La$ requires computing the penalty function
\eqref{eq:gen-omega}. These observations naturally suggests an alternating minimization
algorithm, which has already been considered in special cases in
\cite{AEP}. To describe our algorithm we choose $\epsilon > 0$
and introduce the mapping $\phi^{\epsilon}:\R^n \rightarrow
\R_{++}^n$, whose $i$-th coordinate at $\beta \in \R^n$ is given by $$
\phi_i^{\epsilon}(\beta) = \sqrt{\beta_i^2 + \epsilon}.
$$
For $\beta \in (\R\backslash\{0\})^n$, we also let $\lam(\beta)=
{\rm argmin} \{ \Gamma(\beta,\lam): \lam \in \La\}$.

The alternating minimization algorithm is defined as follows: choose
$\lam_0 \in \La$ and, for $k \in \N$, define the iterates
\begin{eqnarray}
\beta^{k} &=& \beta(\lam^{k-1})
\label{ooo} \\
\lam^{k}& =& \lam(\phi^{\epsilon}(\beta^{k})).
\label{ooo1}
\end{eqnarray}
\label{algo:aa}
The following theorem establishes convergence of this algorithm.
\begin{theorem}
If the set $\La$ is admissible in the sense of Definition \ref{def:1},
then the iterations (\ref{ooo})--(\ref{ooo1})
converges to a vector $\gamma(\epsilon)$ such that $$
\gamma({\epsilon}) = {\rm argmin} \left\{
\|y-X\beta\|^2_2 + 2 \rho \Omega(\phi^{\epsilon}(\beta)|\La)
: \beta \in \R^n\right\}.  $$ Moreover,
any convergent subsequence of
the sequence $\{\gamma\left(\frac{1}{\ell}\right): \ell \in \N\}$
converges to a solution of the optimization problem \eqref{eq:method}.
\label{thm:aa}
\end{theorem}
\begin{proof}
We divide the proof into several steps. To this end, we define
$$
E_\epsilon(\beta,\lam) := \|y-X\beta\|^2_2 + 2 \rho \Gamma(\phi^\epsilon(\beta),\lam)
$$
and note that $\beta(\lam)= {\rm argmin}\{E_\epsilon(\alpha,\lam): \alpha \in \R^n\}$.

{\em Step 1.} We define two sequences, $\theta_k =
E_\epsilon(\beta^{k},\lam^{k-1})$ and $\nu_k =
E_\epsilon(\beta^{k},\lam^{k})$ and observe, for any $k \geq 2$,
that \beq \label{sand} \nu_k \leq \theta_k \leq \nu_{k-1}. \eeq
These inequalities follow directly from the definition of the
alternating algorithm, see equations \eqref{ooo} and \eqref{ooo1}.

{\em Step 2.} We define the compact set $B =\{\beta: \beta \in \R^n,
\|\beta\|_1 \leq \theta_1\}$. From the first inequality in Proposition
\ref{prop:ss} and inequality
\eqref{sand} we conclude, for every $k \in \N$, that $\beta^{k} \in B$.

{\em Step 3.} We define the function $g: \R^n \rightarrow \R$ at $\beta
\in \R^n$ as $$ g(\beta) = \min\left\{E_\epsilon(\alpha,
\lam(\phi^\epsilon(\beta))): \alpha \in \R^n \right\}.  $$ We claim
that $g$ is continuous on $B$. In fact, there exists a constant
$\kappa > 0$ such that, for every $\gamma^1,\gamma^2 \in B$, it holds
that
\beq
\label{eq:pippo1}
|g(\gamma^1)- g(\gamma^2)| \leq \kappa \|\lam(\phi^\epsilon(\gamma^1))- \lam(\phi^\epsilon(\gamma^2))\|_{\infty}.
\eeq
The essential ingredient in the proof of this inequality is the fact that there exists constant ${\bar a}$ and ${\bar b}$
such that, for all $\beta \in B$, $\lam(\phi^\epsilon(\beta)) \in [{\bar a},{\bar b}]^n$. This follows from
the inequalities developed in the proof of Proposition \ref{prop:0}.

{\em Step 4.} By step 2, there exists a subsequence $\{\beta^{k_\ell}: \ell \in \N\}$ which converges to
$\tbeta \in B$ and, for all $\beta \in \R^n$ and $\lam \in \La$, it holds that
\beq
E_\epsilon(\tbeta,\lam(\phi^\epsilon(\tbeta))) \leq E_\epsilon(\beta,\lam(\phi^\epsilon(\tbeta))),~~~~E_\epsilon(\tbeta,\lam(\phi^\epsilon(\tbeta))) \leq E_\epsilon(\tbeta,\lam).
\label{eq:huh}
\eeq
Indeed, from step 1 we conclude that there exists $\psi \in \R_{++}$ such that
$$
\lim_{k \rightarrow \infty} \theta_k = \lim_{k \rightarrow \infty} \nu_k = \psi.
$$
Since, by Proposition \ref{prop:0} $\lam(\beta)$ is continuous for $\beta \in (\R\backslash\{0\})^n$, we obtain that
$$
\lim_{\ell \rightarrow \infty} \lam^{k_\ell} = \lam(\phi^\epsilon(\tbeta)).
$$
By the definition of the alternating algorithm, we have, for all
$\beta \in \R^n$ and $\lam \in \La$, that $$
\theta_{k+1} = E_\epsilon(\beta^{k+1},\lam^k) \leq E_\epsilon(\beta,\lam^k), ~~~~~
\nu_k = E_\epsilon(\beta^k,\lam^k) \leq E_\epsilon(\beta^k,\lam).
$$
From this inequality we obtain, passing to limit, inequalities \eqref{eq:huh}.

{\em Step 5.} The vector $(\tbeta,\lam(\phi^\epsilon(\tbeta)))$ is a stationary point. Indeed, since $\La$ is admissible,
by step 3, $\lam(\phi^\epsilon(\tbeta) \in {\rm int}(\La)$. Therefore, since $E_\epsilon$ is continuously differentiable this claim follows from step 4.

{\em Step 6.} The alternating algorithm converges. This claim follows from the fact that $E_\epsilon$ is strictly convex.
Hence, $E_\epsilon$ has a unique global minimum in $\R^n \times \La$, which in virtue of inequalities \eqref{eq:huh} is attained at
$(\tbeta,\lam(\phi^\epsilon(\tbeta)))$.

The last claim in the theorem follows from the fact that the set $\{\gamma(\epsilon): \epsilon > 0\}$ is bounded and the function $\lam(\beta)$ is continuous.
\end{proof}



The most challenging step in the alternating algorithm is the
computation of the vector $\lam^{k}$. Fortunately, if $\La$ is a
second order cone, problem \eqref{eq:gen-omega} defining the penalty
function $\Omega(\cdot|\La)$ may be reformulated as a second order
cone program (SOCP), see e.g. \cite{boyd}. To see this, we introduce
an additional variable $t
\in \R^n$ and note that
$$
\Omega(\beta|\La) = \min \left\{ \sum_{i \in \N_n} t_i + \lam_i: \|(2\beta_i,t_i -\lam_i)\|_2 \leq t_i + \lam_i, t_i \geq 0,~i \in \N_n,~\lam \in \La\right\}.
$$ In particular, the examples discussed in Sections \ref{sec:WP} and
\ref{sec:GP}, the set $\La$ is formed by linear constraints and, so,
problem \eqref{eq:gen-omega} is an SOCP. We may then use available
tool-boxes to compute the solution of this problem. However, in
special cases the computation of the penalty function may be
significantly facilitated by using available analytical formulas.
Here, for simplicity we describe how to do this in the case of the
wedge penalty. For this purpose we say that a vector $\beta \in
\R^n$ is admissible if, for every $k \in \N_{n}$, it holds that $\|\beta_{|\N_k}\|_2/\sqrt{k} \leq \|\beta\|_2/\sqrt{n}$.

The proof of the next lemma is straightforward and we do not
elaborate on the details.
\begin{lemma}
If $\beta \in \R^n$ and $\delta \in \R^p$ are admissible and
$\|\beta\|_2/\sqrt{n} \leq \|\delta\|_2/\sqrt{p}$ then $(\beta,\delta)$ is
admissible.
\label{lemma:admissible}
\end{lemma}

\begin{figure}[t]
  \begin{center}
   \begin{tabular}{l}
    \hline
    \textbf{Initialization:} $k \leftarrow 0$ \\
    \textbf{Input:} $\beta \in \mathbb{R}^n$;~~~ \textbf{Output:} $J_1, \ldots, J_k$ \\
    \textbf{for} $t =1$ \textbf{to} $n$ \textbf{do} \\
    \quad $J_{k+1} \leftarrow \lbrace t \rbrace$; \\ \quad $k \leftarrow k + 1$ \\
    \quad \textbf{while} $k > 1$ \textbf{and}
    $\frac{\|\beta_{|J_{k-1}}\|_2}{\sqrt{|J_{k-1}|}} \leq \frac{\|\beta_{|J_k}\|_2}{\sqrt{|J_k|}}$ \vspace{.08truecm}\\
    \quad \quad $J_{k-1} \leftarrow J_{k-1} \cup J_k$ \\
\quad \quad $k \leftarrow k-1$ \\
    \quad \textbf{end} \\
    \textbf{end} \\
    \hline
    \end{tabular}
  \end{center}
  \caption{Iterative algorithm to compute the wedge penalty}
  \label{fig:partAlg}
\end{figure}

The iterative algorithm presented in Figure \ref{fig:partAlg} can be
used to find the partition $\mathcal{J} = \lbrace J_\ell: \ell \in \N_k \rbrace$ and, so, the
vector $\lam(\beta)$ described in Theorem \ref{thm:sp}.
The algorithm processes the components of vector $\beta$ in a
sequential manner. Initially, the first component forms the only set
in the partition. After the generic iteration $t-1$, where the
partition is composed of $k$ sets, the index of the next components,
$t$, is put in a new set $J_{k+1}$. Two cases can occur: the means of
the squares of the sets are in strict descending order, or this order
is violated by the last set. The latter is the only case that requires
further action, so the algorithm merges the last two sets and repeats
until the sets in the partition are fully ordered. Note that, since the only operation performed by the algorithm is the merge of admissible sets, Lemma \ref{lemma:admissible}
ensures that after each step $t$ the current partition satisfies the ``stay within'' conditions
$\frac{\|\beta_{|{K}}\|_2}{\sqrt{k}} > \frac{\|\beta_{|{J_\ell \backslash K}}\|_2}{\sqrt{|J_\ell|-k}}$,
for every $\ell \in \N_k$ and every subset $K \subset J_\ell$ formed by the first $k < |J_\ell|$
elements of $J_\ell$. Moreover, the \emph{while} loop ensures that after each step the current partition satisfies,
for every $\ell \in \N_{k-1}$, the ``cross over'' conditions
$\|\beta_{|J_\ell}\|_2\sqrt{|J_\ell|} > \|\beta_{|J_{\ell+1}}\|_2\sqrt{|J_{\ell+1}|}$.
Thus, the output of the algorithm is the partition $\mathcal{J}$ defined in Theorem \ref{thm:sp}.
In the actual implementation of the algorithm, the means of squares of each set can
be saved. This allows us to compute the mean of squares of a merged
set as a weighted mean, which is a constant time operation. Since
there are $n-1$ consecutive terms in total, this is also the maximum
number of merges that the algorithm can perform. Each merge requires
exactly one additional test, so we can conclude that the running time
of the algorithm is linear.

\section{Numerical simulations}
\label{sec:exp}
In this section we present some numerical simulations with the
proposed method. For simplicity, we consider data generated
noiselessly from $y = X
\beta^*$, where $\beta^* \in \mathbb{R}^{100}$ is the true underlying
regression vector, and $X$ is an $m\times 100$ input matrix, $m$
being the sample size. The elements of $X$ are generated i.i.d.
from the standard normal distribution, and the columns of $X$
are then normalized such that their $\ell_2$ norm is $1$. Since we
consider the noiseless case, we solve the interpolation problem $\min
\{\Omega(\beta) : y = X\beta\}$, for different choices of the penalty
function $\Omega$. In practice,
\eqref{eq:method} is solved for a tiny value of the parameter, for example, $\rho=10^{-8}$, which
we found to be sufficient to ensure that the error term in \eqref{eq:method} is negligible at the minimum. All
experiments were repeated $50$ times, generating each time a new
matrix $X$. In the figures we report the average of the model error of
the vector $\hat{\beta}$ learned by each method, as a function of the sample size $m$.
The former is defined as $\text{ME}(\hat{\beta}) =
\mathbb{E}[\|\hat{\beta} -
\beta^\ast\|^2_2] / \mathbb{E}[\|\beta^\ast\|^2_2]$. In the following, we discuss a series of experiments, corresponding
to different choices for the model vector $\beta^*$ and its sparsity
pattern. In all experiments, we solved the optimization problem
\eqref{eq:method} with the algorithm presented in Section
\ref{sec:algo}.  Whenever possible we solved step \eqref{ooo1} using
analytical formulas and resorted to the solver CVX ({\em
http://cvxr.com/cvx/}) in the other cases.  For example, in the case
of the wedge penalty, we found that the computational time of the
algorithm in Figure \ref{fig:partAlg} is $495,603,665,869,1175$ faster
than that of the solver CVX for $n=100,500,1000,2500,5000$,
respectively. Our implementation ran on a 16GM memory dual core Intel
machine. The MATLAB code is available at
http://www.cs.ucl.ac.uk/staff/M.Pontil/software.html.

\begin{figure}[t]
\begin{center}
  \begin{tabular}{cccccc}
    \includegraphics[width=0.38\textwidth]{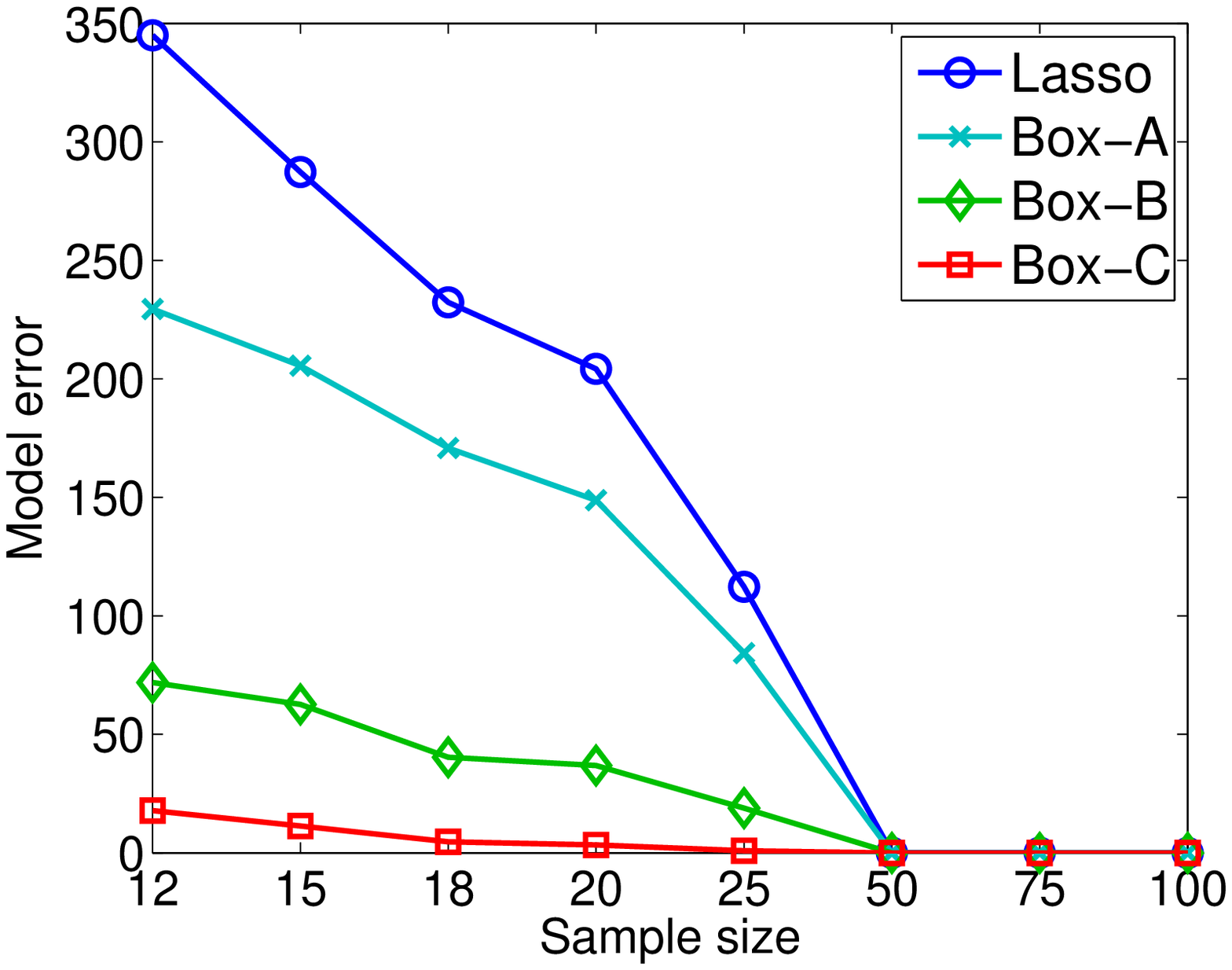} &
    \includegraphics[width=0.38\textwidth]{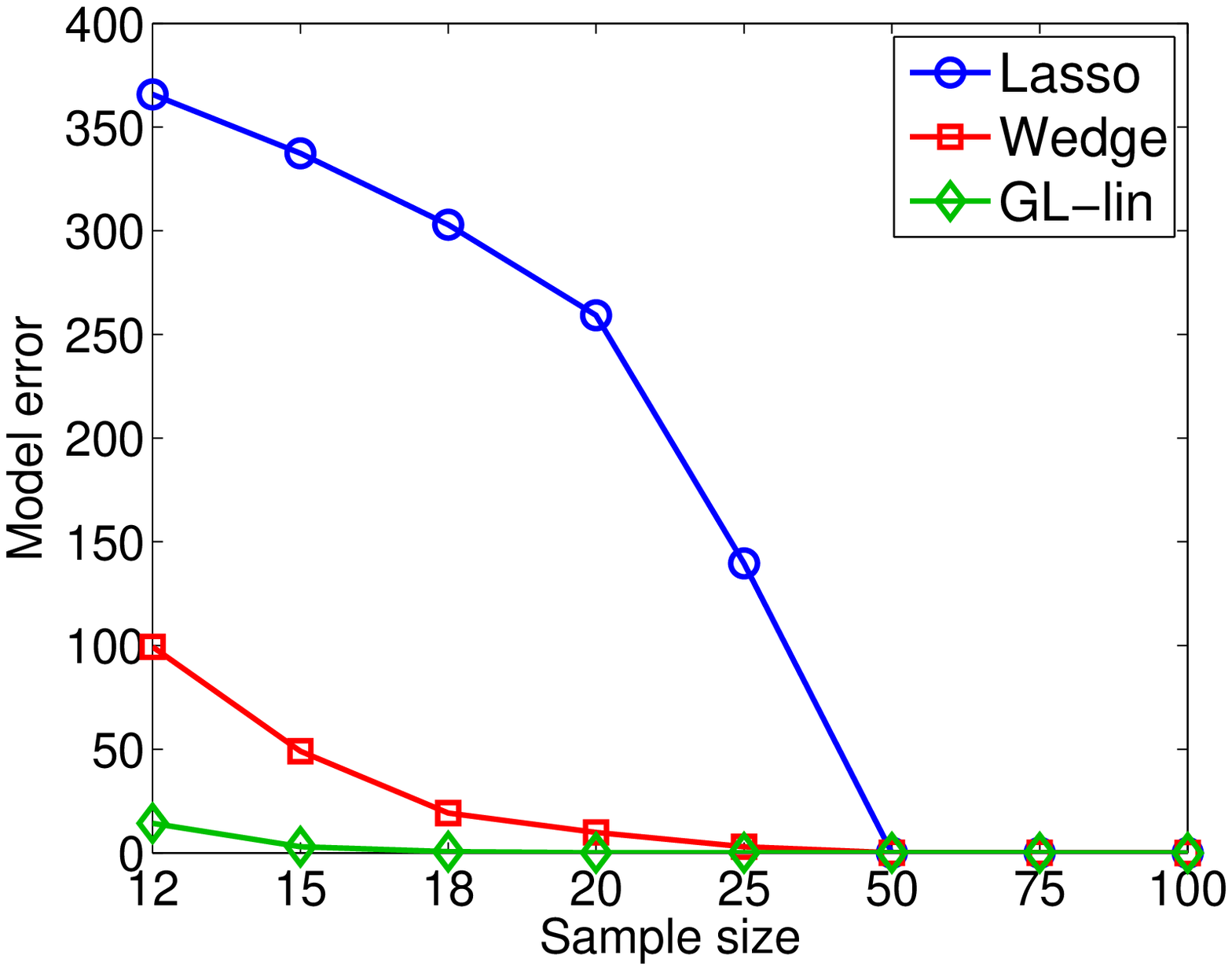} \\
    (a) & (b) \\
     \includegraphics[width=0.38\textwidth]{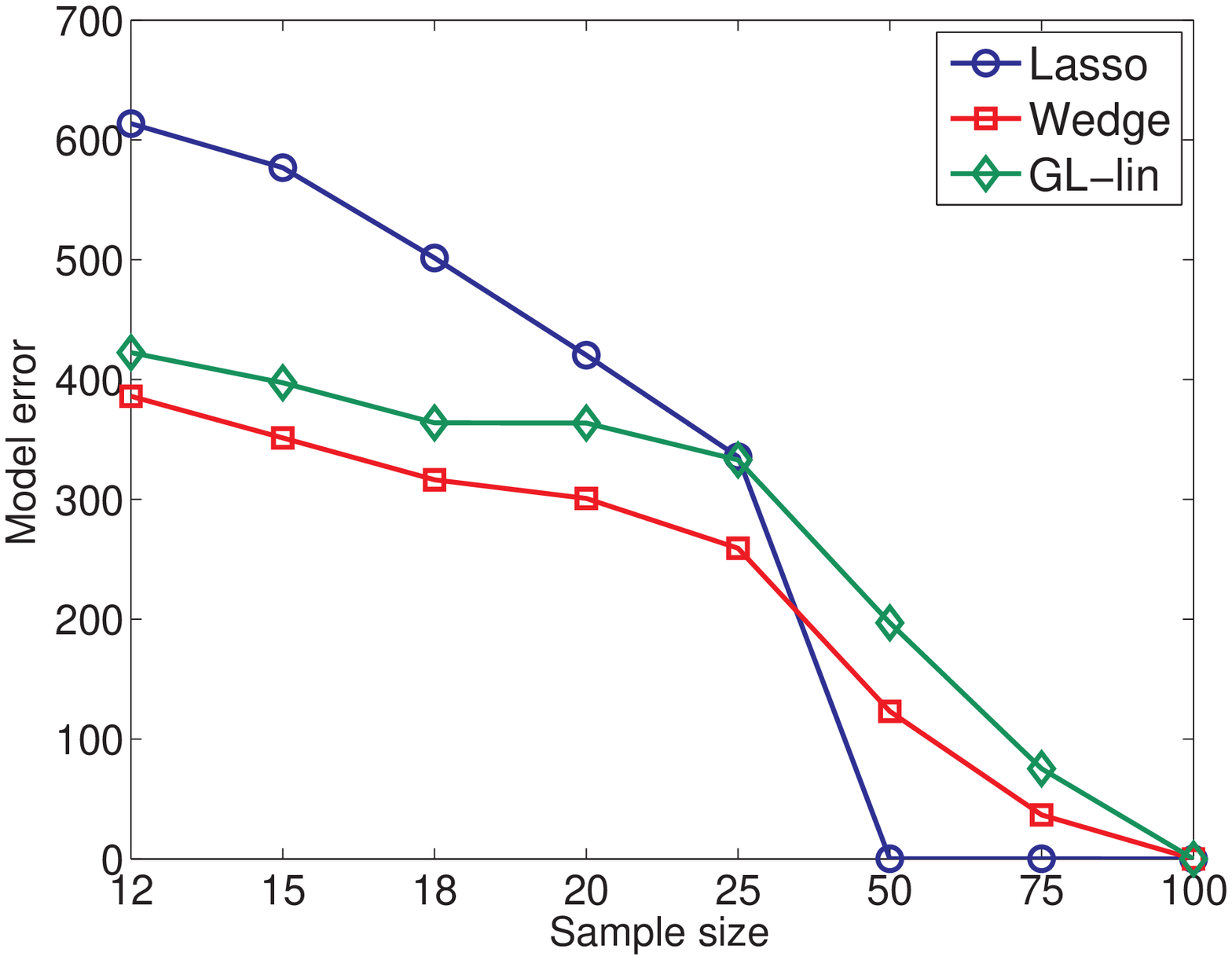} &
    \includegraphics[width=0.38\textwidth]{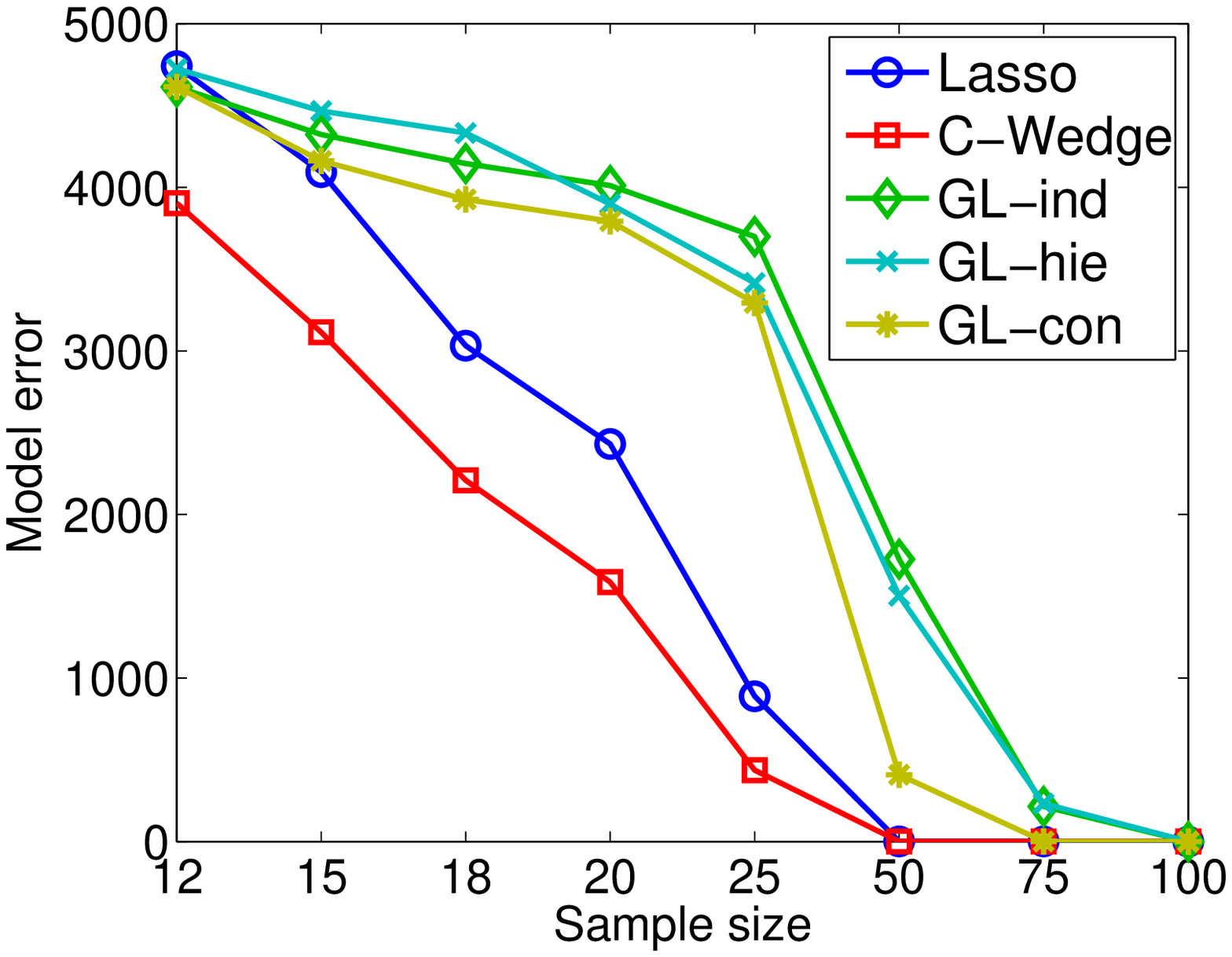} \\
(c) & (d) \\
  \end{tabular}
  \caption {Comparison between different penalty methods: (a) Box vs. Lasso; (b,c) Wedge vs. Hierarchical group Lasso; (d)
Composite wedge. See text for more information}
  \label{fig:all}
\end{center}
\end{figure}

\vspace{.2truecm}
\noindent {\bf Box.} In the first experiment the model is $10$-sparse, where
each nonzero component, in a random position, is an integer uniformly sampled in
the interval $[-10, 10]$. We wish to show that the more
accurate the prior information about the model is, the more precise
the estimate will be. We use a box penalty (see Theorem \ref{thm:box})
constructed ``around'' the model, imagining that an oracle tells us
that each component $|\beta_i^\ast|$ is bounded within an
interval. We consider three boxes $B[a, b]$ of different sizes, namely
$a_i = (r-|\beta_i^\ast|)_+$ and $b_i = (|\beta_i^\ast| - r)_+$ and
radii $r = 5, 1$ and $0.001$, which we denote as Box-A, Box-B and
Box-C, respectively. We compare these methods with the Lasso -- see Figure \ref{fig:all}-a.
As expected, the three box penalties perform better. Moreover,
as the radius of a box diminishes, the amount of information about the
true model increases, and the performance improves.


\vspace{.2truecm}
\noindent{\bf Wedge.}
In the second experiment, we consider a regression vector, whose components are
nonincreasing in absolute value and only a few are nonzero.
Specifically, we choose a $10$-sparse vector: $\beta^*_j = 11-j$, if $j \in \N_{10}$ and zero otherwise.
We compare the Lasso, which makes no use of such ordering information, with the
wedge penalty $\Omega(\beta|W)$ (see Theorem \ref{thm:sp})
and the hierarchical group Lasso in \cite{binyu},
which both make use of such information. For the group Lasso we choose
$\Omega(\beta) = \sum_{\ell \in \N_{100}} \| \beta_{|{J}_\ell}\|_2$,
with ${J}_\ell = \{\ell,\ell+1,\dots,100\}$, $\ell \in \N_{100}$.
These two methods are referred to as ``Wedge'' and ``GL-lin'' in Figure \ref{fig:all}-b, respectively.
As expected both methods improve over the Lasso, with ``GL-lin'' being the best of the two.
We further tested the robustness of the methods, by adding two additional nonzero components with value of $10$
to the vector $\beta^*$
in a random position between $20$ and $100$. This result, reported in Figure \ref{fig:all}-c, indicates that
``GL-lin'' is more sensitive to such a perturbation.

\vspace{.2truecm}
\noindent{\bf Composite wedge.}
Next we consider a more complex experiment, where the regression vector is sparse within different contiguous regions
$P_1,\dots,P_{10}$,
and the $\ell_1$ norm on one region is larger than the $\ell_1$ norm on the next region.
We choose sets $P_i = \{10(i-1) + 1, \ldots, 10i\}$,
$i\in\mathbb{N}_{10}$ and generate a $6$-sparse vector $\beta^\ast$
whose $i$-th nonzero element has value $31-i$ (decreasing) and
is in a random position in $P_i$, for $i
\in\mathbb{N}_{6}$. We encode this prior knowledge by choosing $\Omega(\beta|\La)$ with $\Lambda = \left \lbrace \lambda \in
\mathbb{R}^{100} : \|\lambda_{P_i}\|_1 \geq \|\lambda_{P_{i+1}}\|_1,~i \in \N_{9}
\right \rbrace$. This method constraints the sum of the sets to be nonincreasing and
may be interpreted as the composition of the wedge set with an average
operation across the sets $P_i$, which may be computed using Proposition \ref{prop:comb}
.  This method, which is referred to as ``C-Wedge''
in Figure \ref{fig:all}-d, is compared to the Lasso and to three other
versions of the group Lasso.  The first is a standard group Lasso with
the nonoverlapping groups $J_i = P_i$, $i \in \N_{10}$, thus
encouraging the presence of sets of zero elements, which is useful
because there are $4$ such sets. The second is a variation of the
hierarchical group Lasso discussed above with $J_i = \cup_{j=i}^{10}
P_j$, $i
\in \mathbb{N}_{10}$. A problem with these approaches is that the $\ell_2$ norm is applied at
the level of the individual sets $P_i$, which does not promote sparsity within these sets. To
counter this effect we can enforce contiguous nonzero patterns within
each of the $P_i$, as proposed by \cite{Jenatton}. That is, we
consider as the groups the sets formed by all sequences of $q \in
\N_9$ consecutive elements at the beginning or at the end of each of the sets
$P_i$, for a total of $180$ groups.  These three groupings will be referred to as ``GL-ind'', ``GL-hie'`, ``GL-con''
in Figure \ref{fig:all}-d, respectively. This result indicates the advantage of ``C-Wedge'' over the other methods
considered. In particular, the group Lasso methods fall behind our method and the Lasso,
with ``GL-con'' being slightly better than ``GL-ind'' and ``GL-hie''.
Notice also that all group Lasso methods gradually diminish the model
error until they have a point for each dimension, while our method and the Lasso
have a steeper descent, reaching zero at a number of points which is less than half the number of dimensions.

\begin{figure}[t]
\begin{center}
  \begin{tabular}{cc}
    \includegraphics[width=0.38\textwidth]{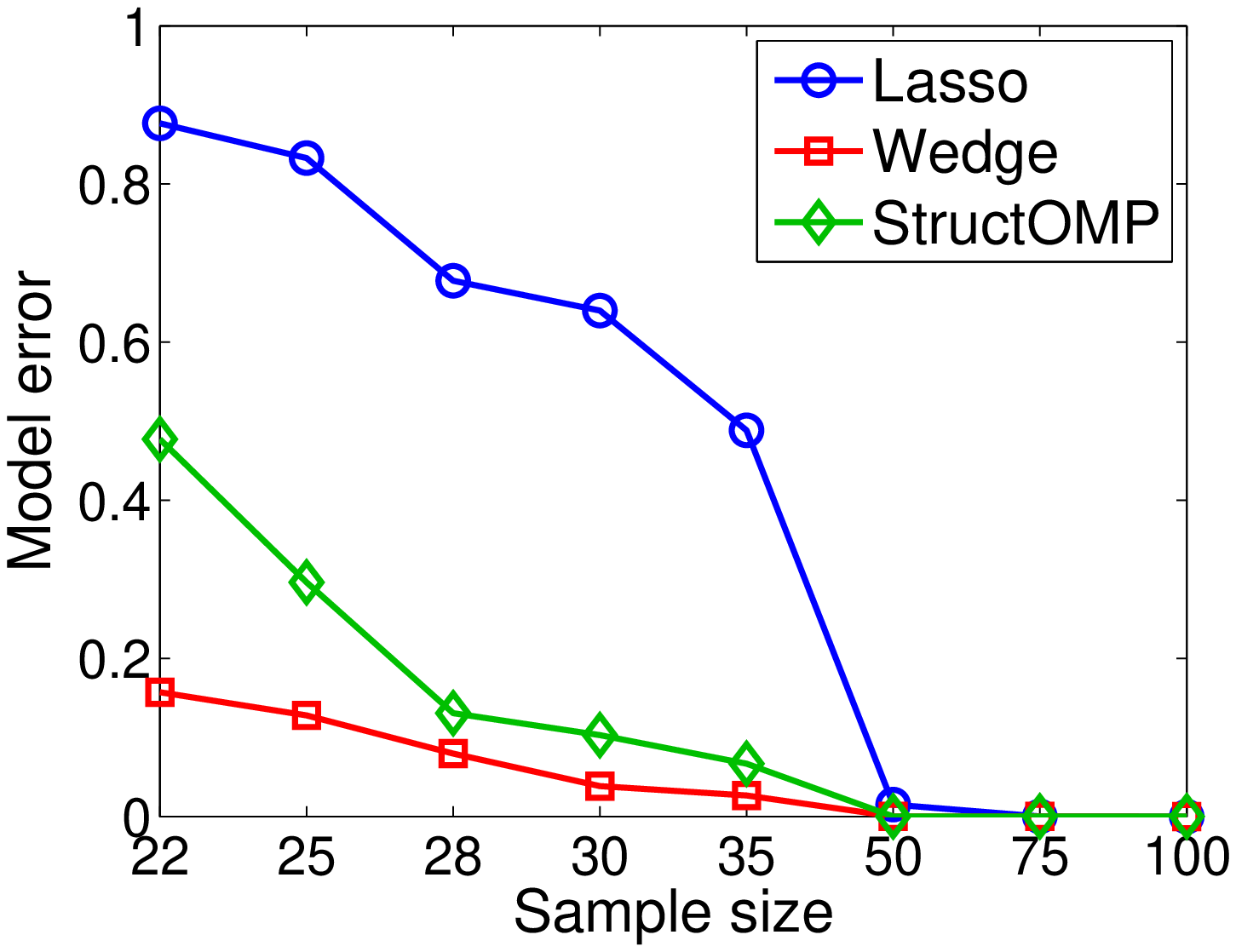} &
    \includegraphics[width=0.38\textwidth]{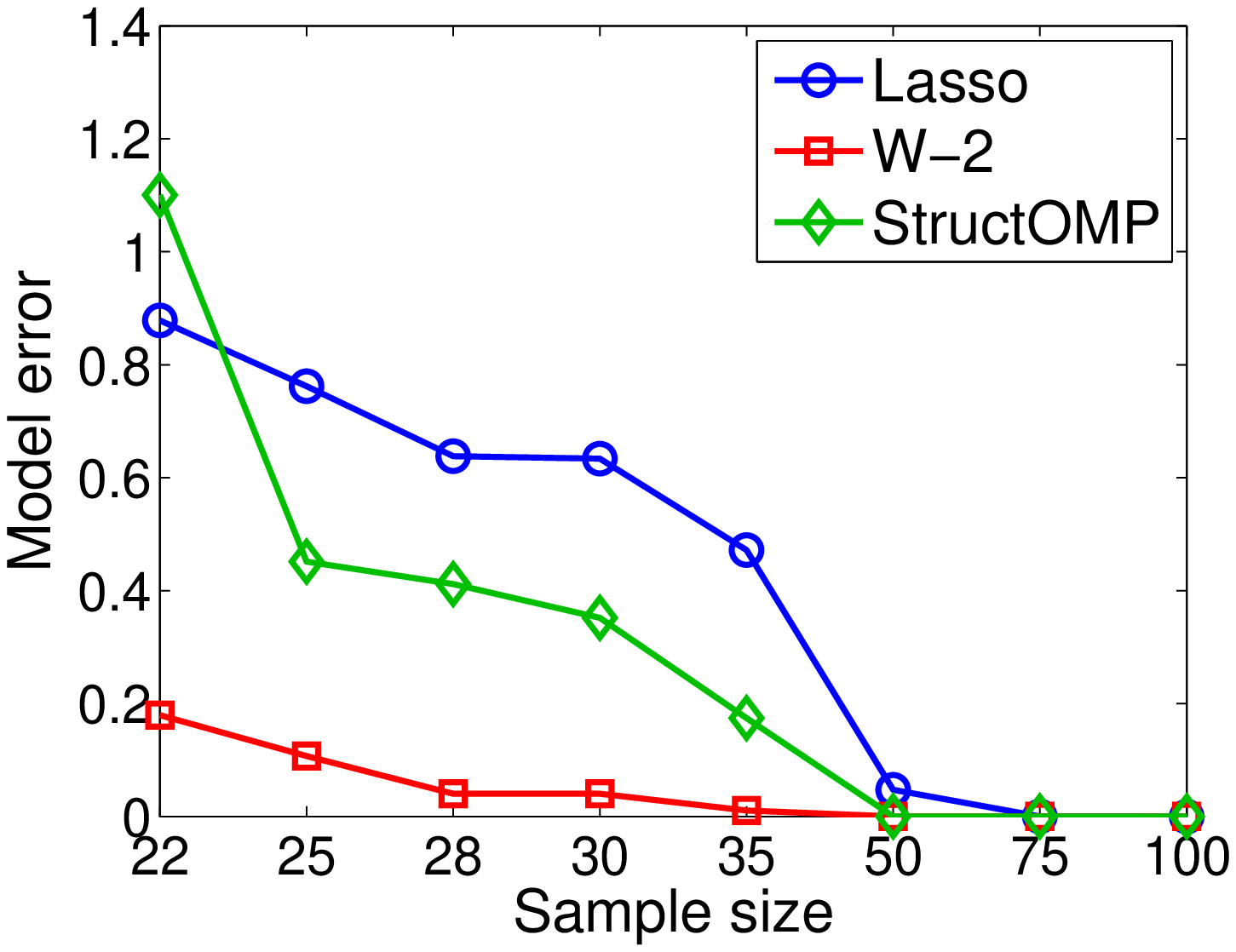} \\
    $(a)$ & $(b)$ \\
    \includegraphics[width=0.38\textwidth]{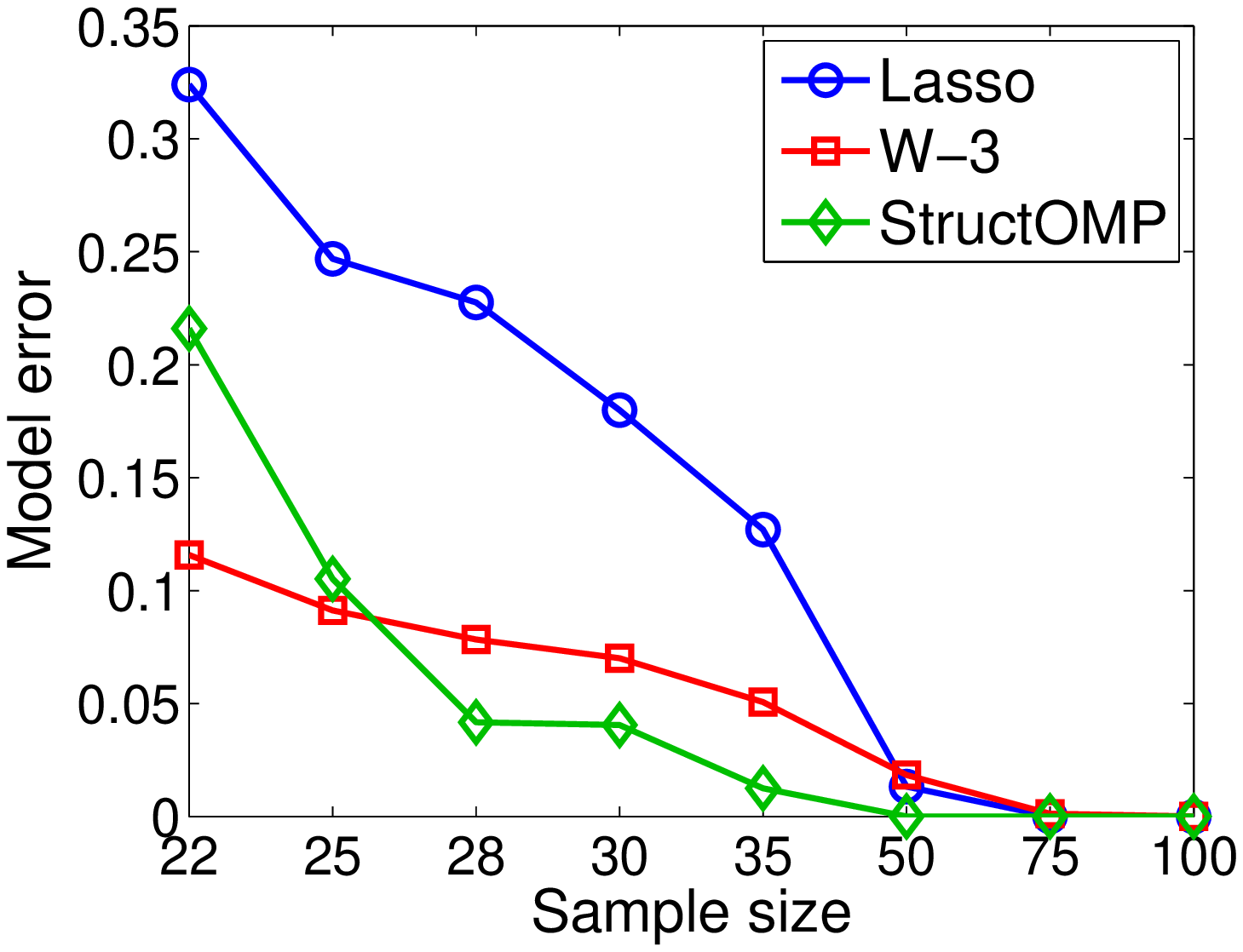} &
    \includegraphics[width=0.38\textwidth]{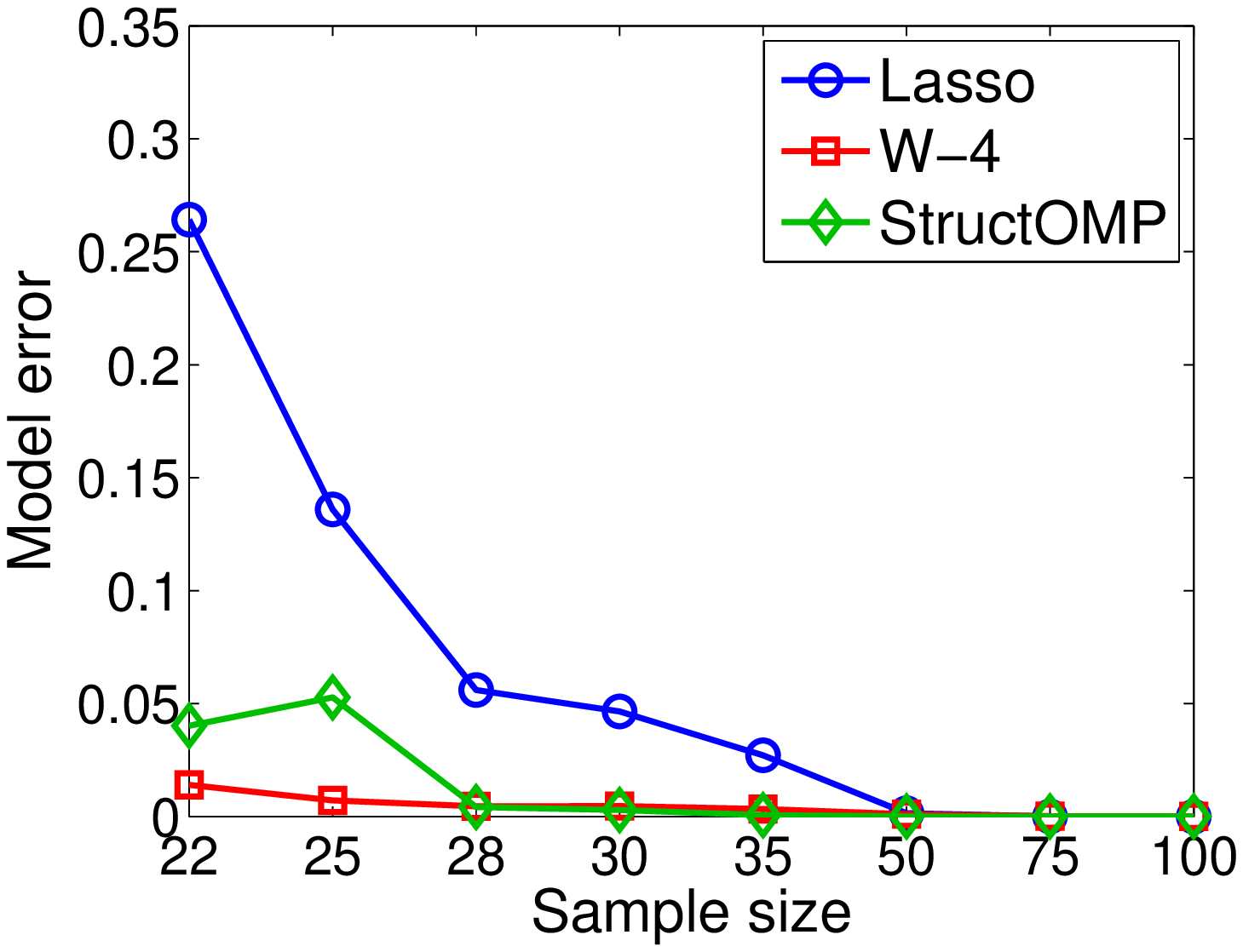} \\
    $(c)$ & $(d)$ \\
  \end{tabular}
  \caption {Comparison between StructOMP and penalty
  $\Omega(\beta|W^k)$, $k=1, \ldots, 4$, used for several polynomial
  models: $(a)$ degree $1$, $(b)$ degree $2$, $(c)$ degree $3$; $(d)$
  degree $4$.}
  \label{fig:polynomials}
  \end{center}
\end{figure}

\vspace{.2truecm}
\noindent
\textbf{Polynomials}. The constraints on the finite differences (see
equation \eqref{eq:k-wedge}) impose a structure on the sparsity of the model. To
further investigate this possibility we now consider some models whose
absolute value belong to the sets of constraints $W^k$, where $k=1,
\ldots, 4$. Specifically, we evaluate the polynomials $p_1(t) =
-(t+5)$, $p_2(t) = (t+6)(t-2)$, $p_3(t) = -(t+6.5)t(t-1.5)$ and
$p_4(t) = (t+6.5)(t-2.5)(t+1)t$ at $100$ equally spaced ($0.1$)
points starting from $-7$. We take the positive part of each component
and scale it to $10$, so that the results can be seen in Figure
\ref{fig:silhouette.poly}. The roots of the polynomials has been chosen
so that the sparsity of the models is either $18$ or $19$.

We solve the interpolation problem using our method with the penalty
$\Omega(\beta | W^k)$, $k = 1, \ldots, 4$, with the objective of
testing the robustness of our method: the constraint set $W^k$ should
be a more meaningful choice when $|\beta^\ast|$ is in it, but the exact
knowledge of the degree is not necessary. This is indeed the case:
``W-k'' outperforms the Lasso for every $k$, but among these methods
the best one ``knows'' the degree of $|\beta^\ast|$. For clarity, in
Figures \ref{fig:polynomials} we included only the best method.

One important feature of these sparsity patterns is the number of
contiguous regions: $1$, $2$, $2$ and $3$ respectively. This prior
information cannot be exploited with convex optimization techniques, so
we tested our method against StructOMP, proposed by \cite{huang2009},
a state of the art greedy algorithm. It relies on a complexity
parameter which depends on the number of contiguous regions of the
model, and which we provide exactly to the algorithm. The performance
of ``W-k'' is comparable or better than StructOMP.

\begin{figure}[th]
\begin{center}
  \begin{tabular}{cc}
    \includegraphics[width=75mm,height=28.08mm]{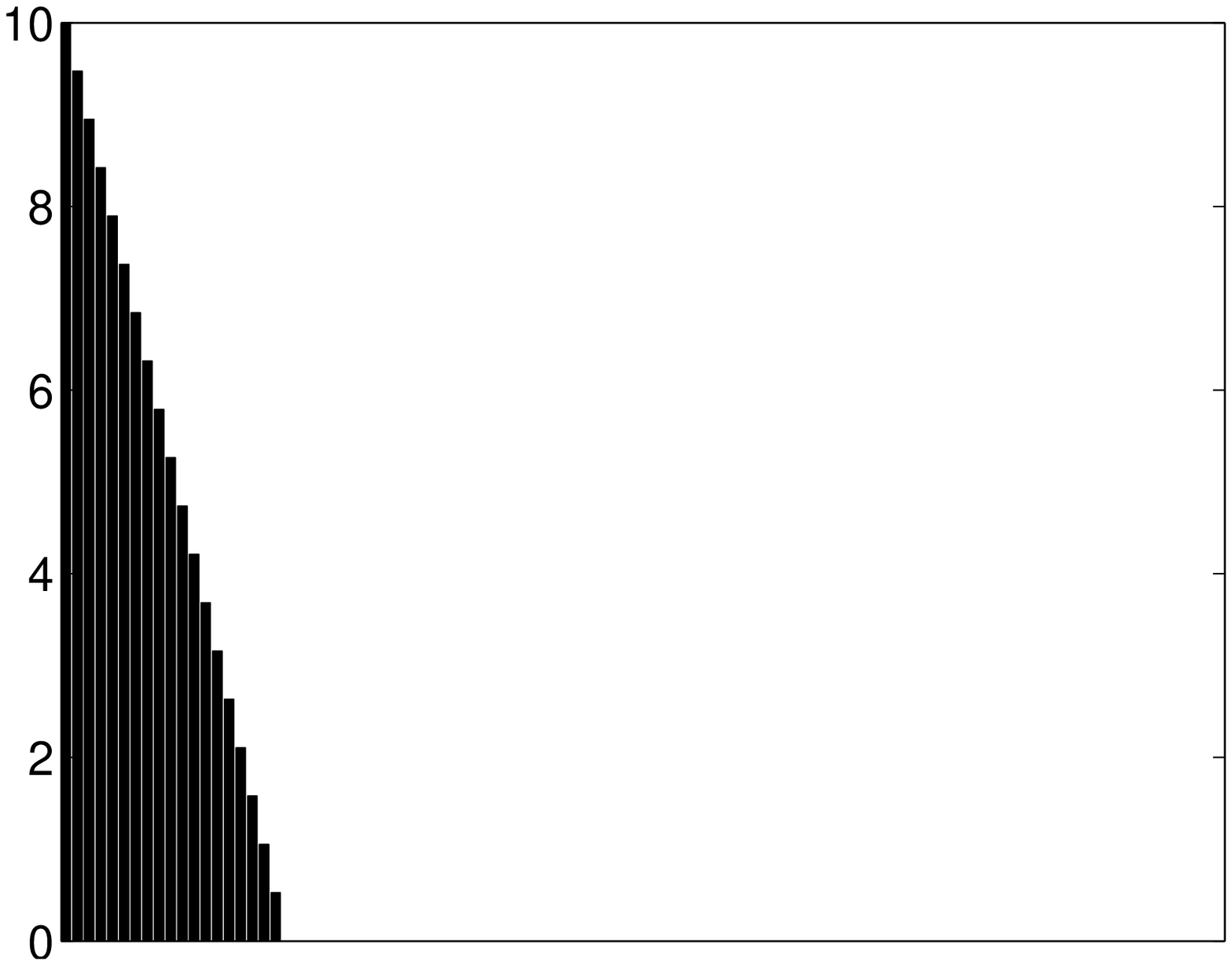} &
    \includegraphics[width=75mm,height=28.08mm]{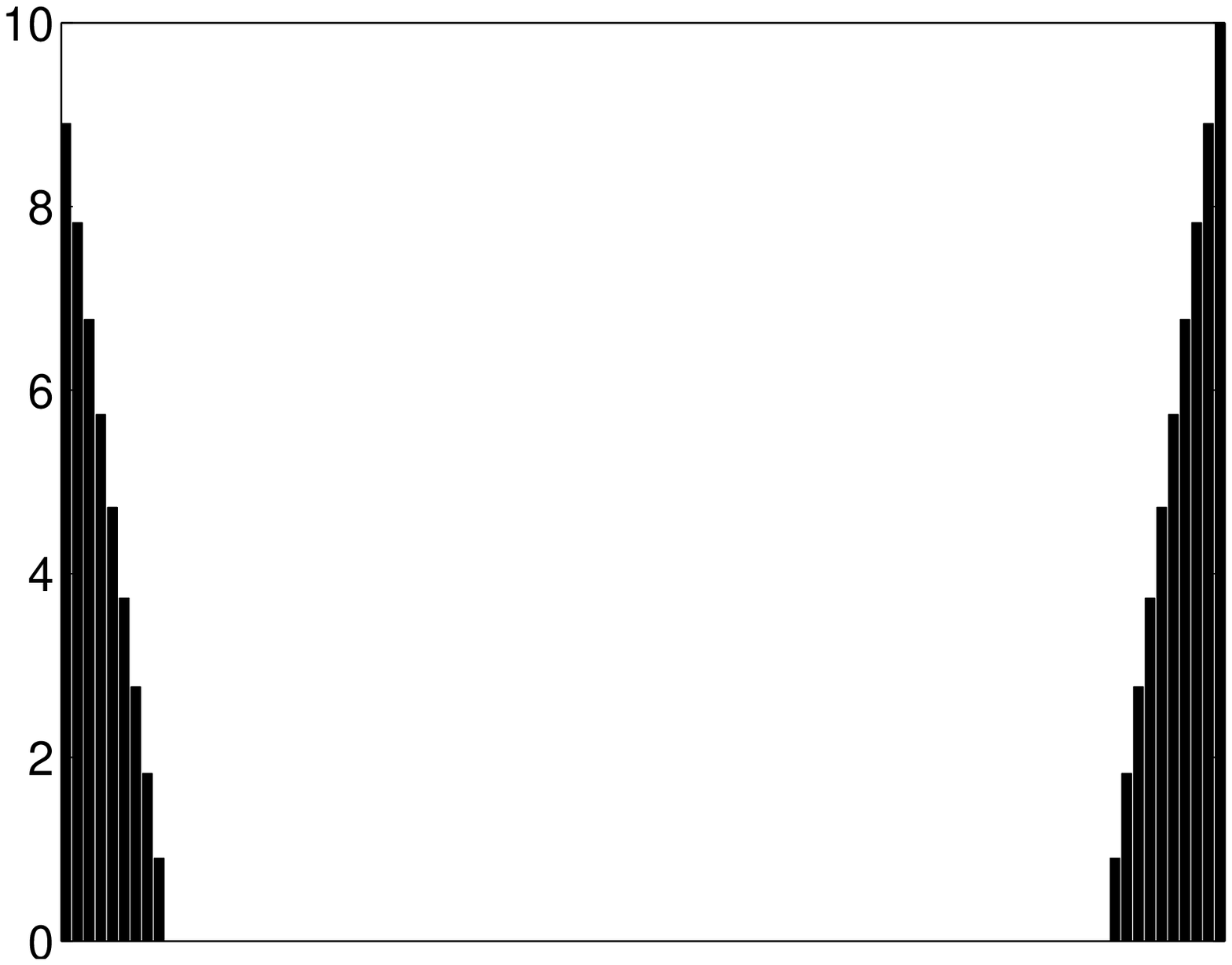} \\
    $(a)$ & $(b)$ \\
    \includegraphics[width=75mm,height=28.08mm]{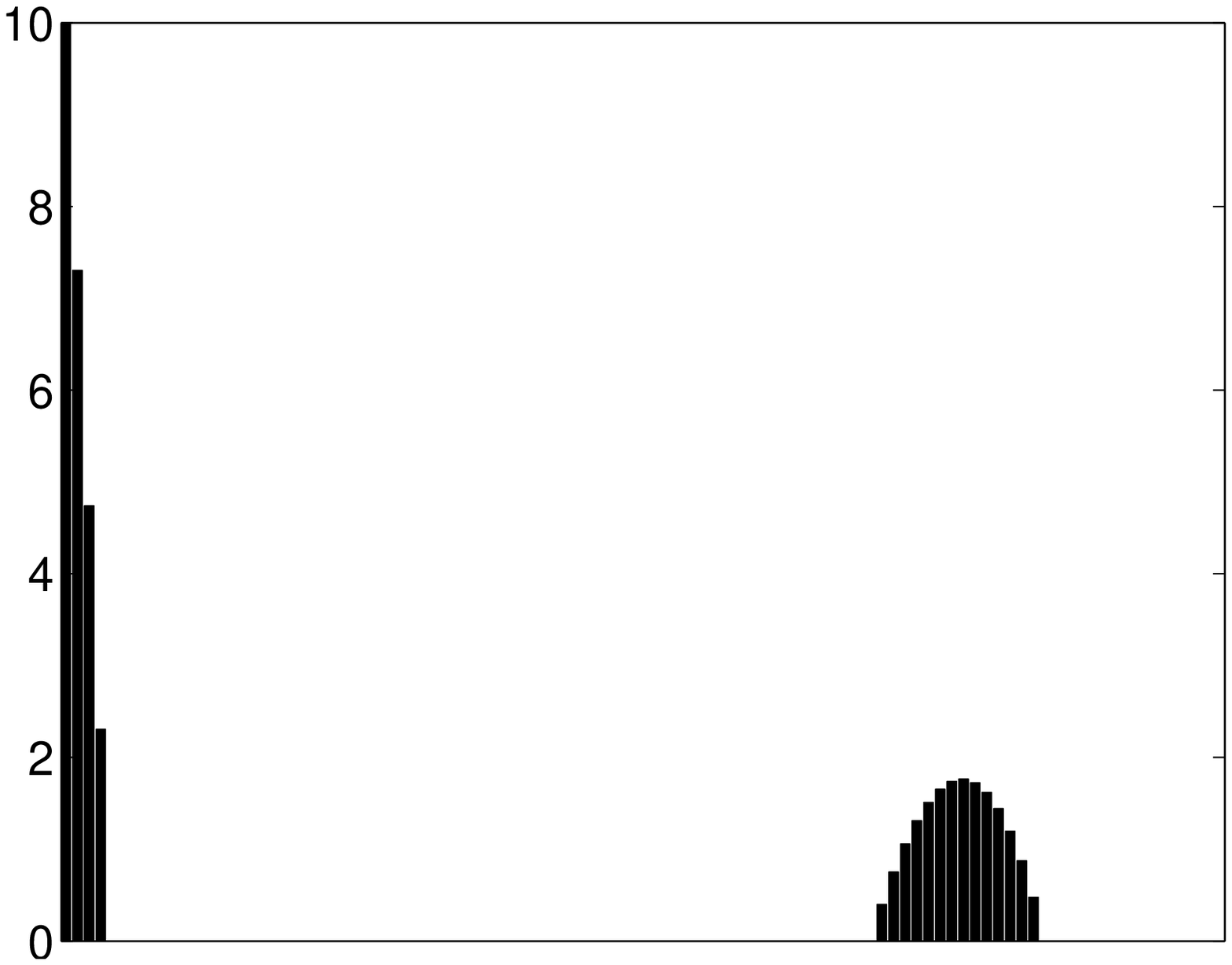} &
    \includegraphics[width=75mm,height=28.08mm]{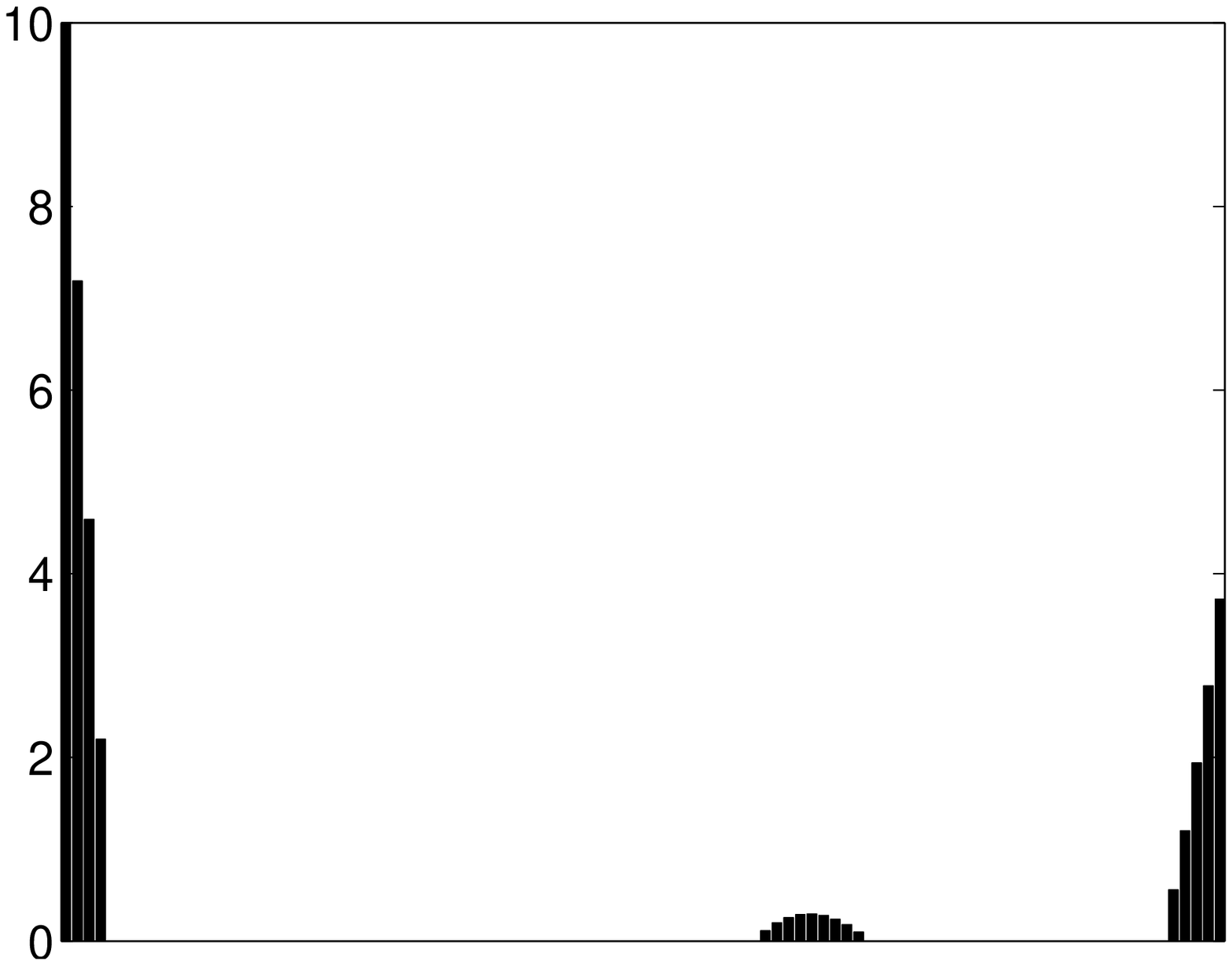} \\
    $(c)$ & $(d)$ \\
  \end{tabular}
  \caption {Silhouette of the polynomials by number of degree: $(a)$
  $k=1$, $(b)$ $k=2$, $(c)$ $k=3$, $(d)$ $k=4$.}
  \label{fig:silhouette.poly}
\end{center}
\end{figure}

As a way of testing the methods on a less artificial setting, we
repeat the experiment using the same sparsity patterns, but replacing
each nonzero component with a uniformly sampled random number between
$1$ and $2$. In Figure \ref{fig:poly.unif} we can see that, even if
now the models manifestly do not belong to $W^k$, we still have an
advantage because the constraints look for a limited number of
contiguous regions. We found that in this case StructOMP has difficulties, probably due to the randomness of the model.

\begin{figure}[th]
\begin{center}
  \begin{tabular}{cc}
    \includegraphics[width=0.38\textwidth]{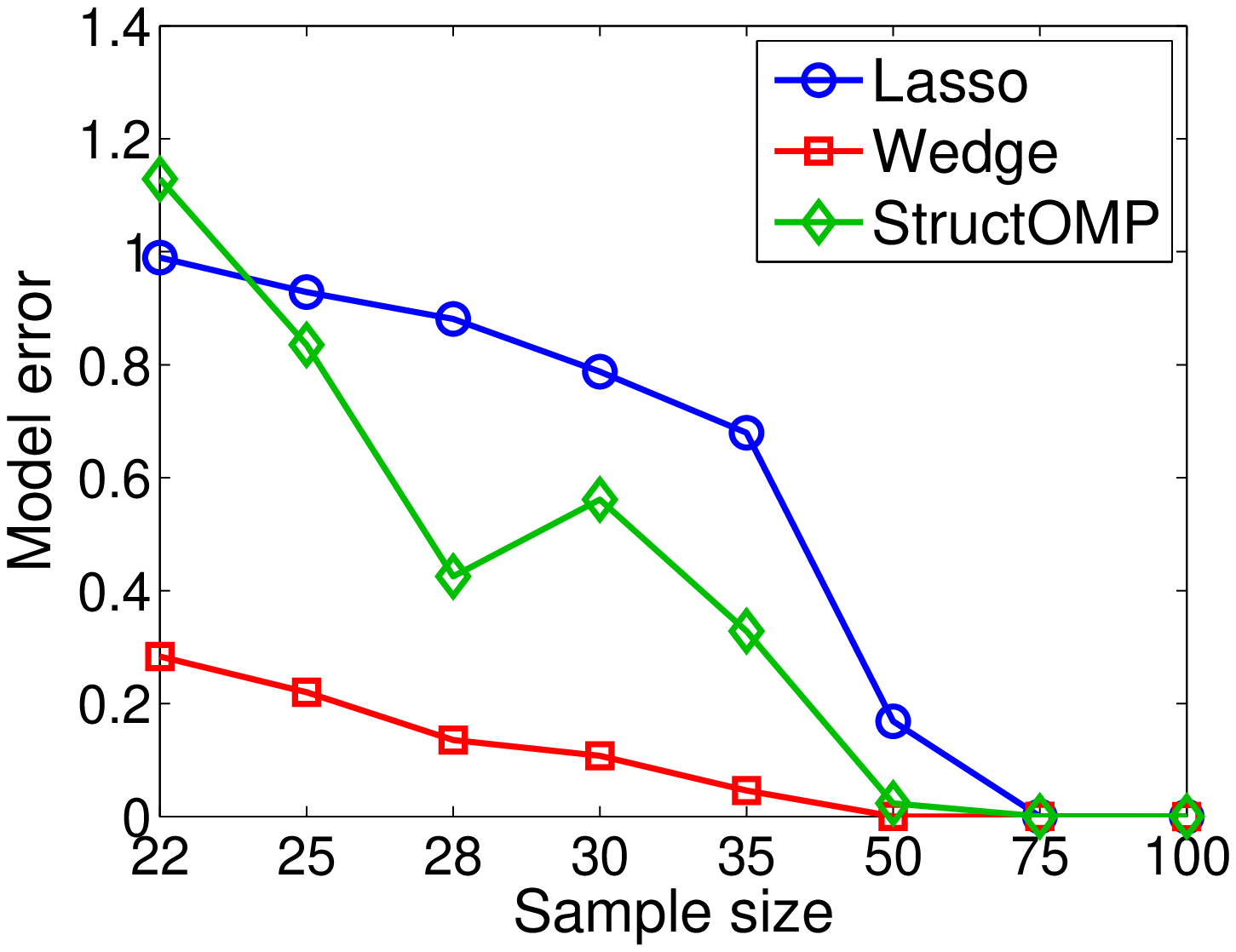} &
    \includegraphics[width=0.38\textwidth]{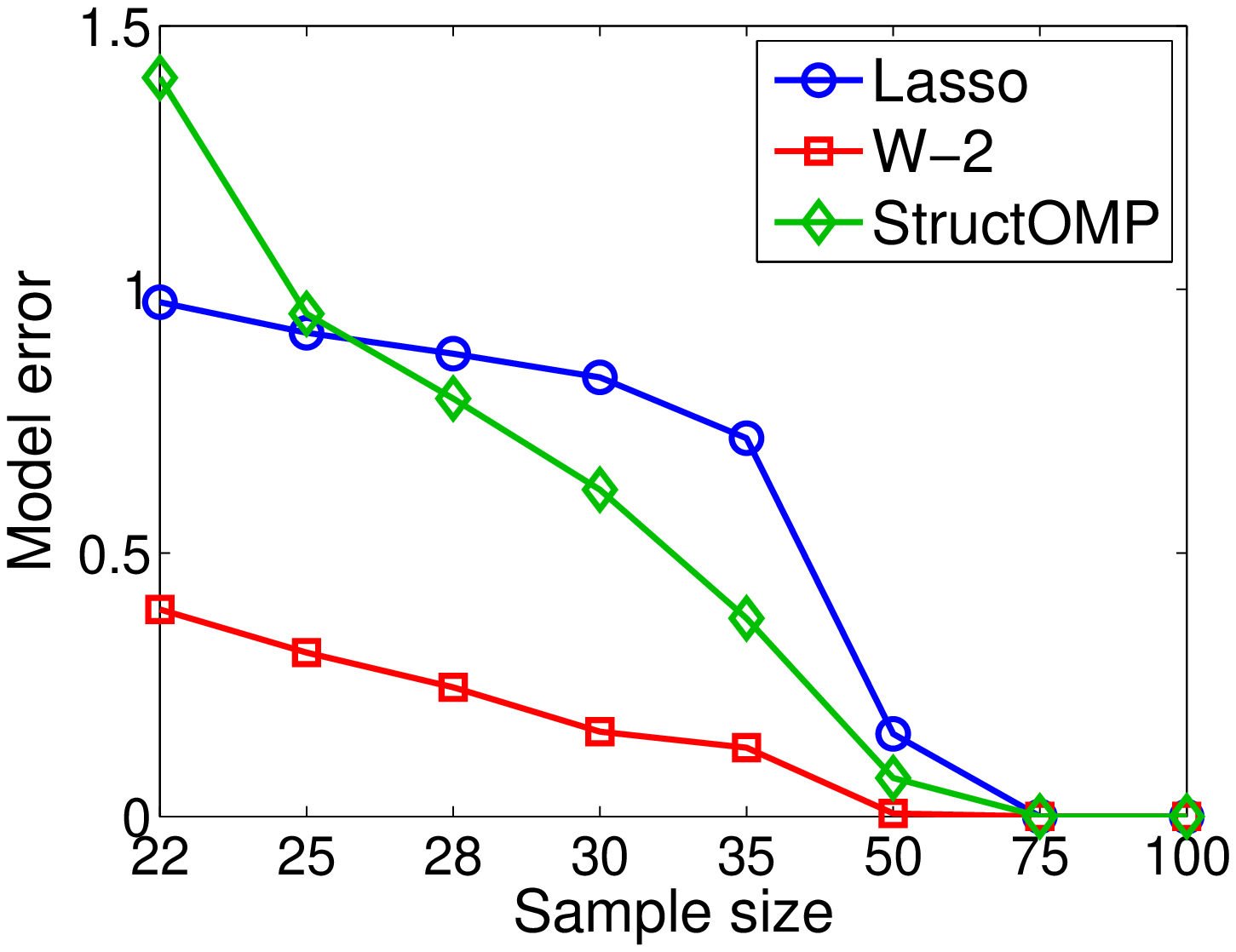} \\
    $(a)$ & $(b)$ \\
    \includegraphics[width=0.38\textwidth]{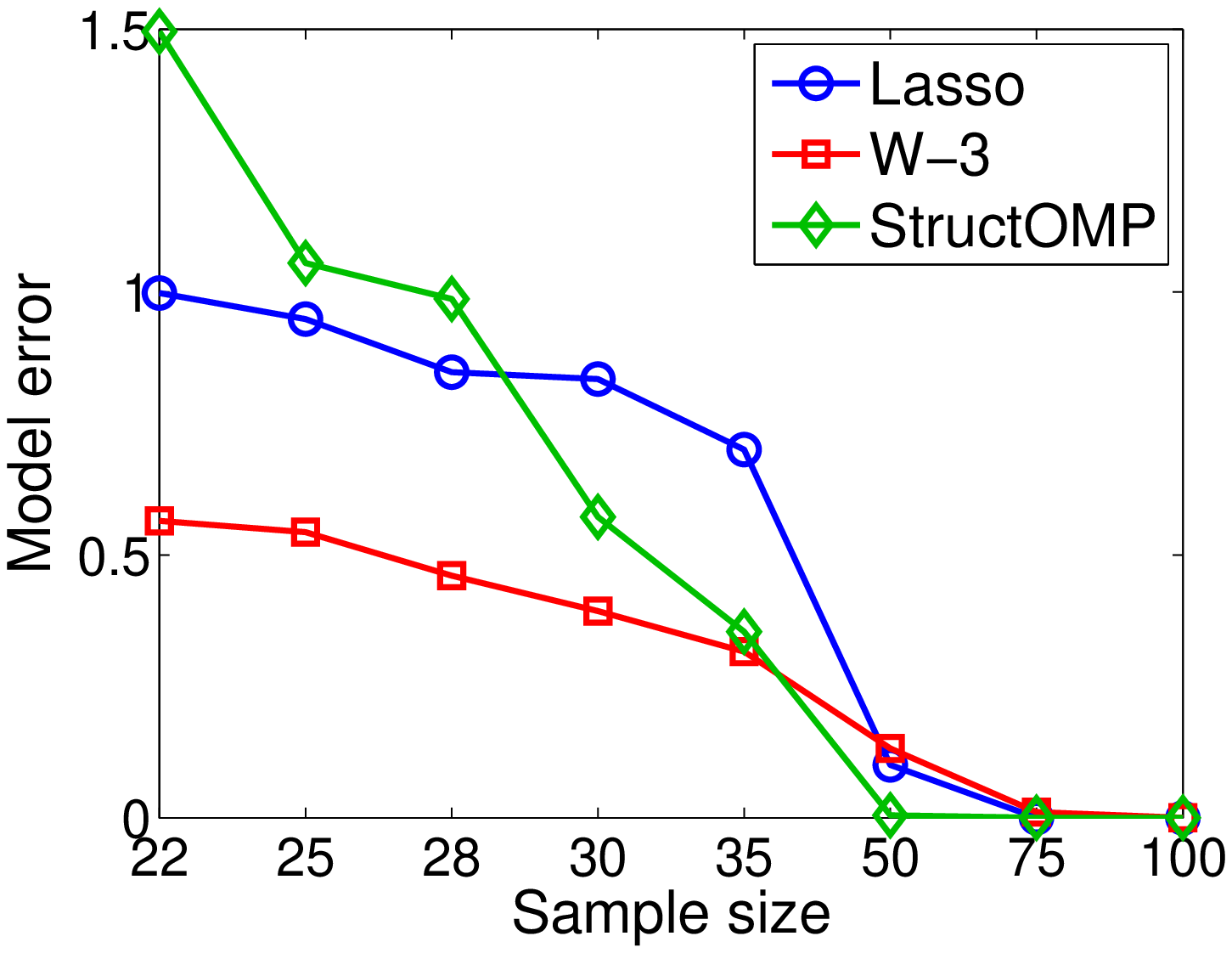} &
    \includegraphics[width=0.38\textwidth]{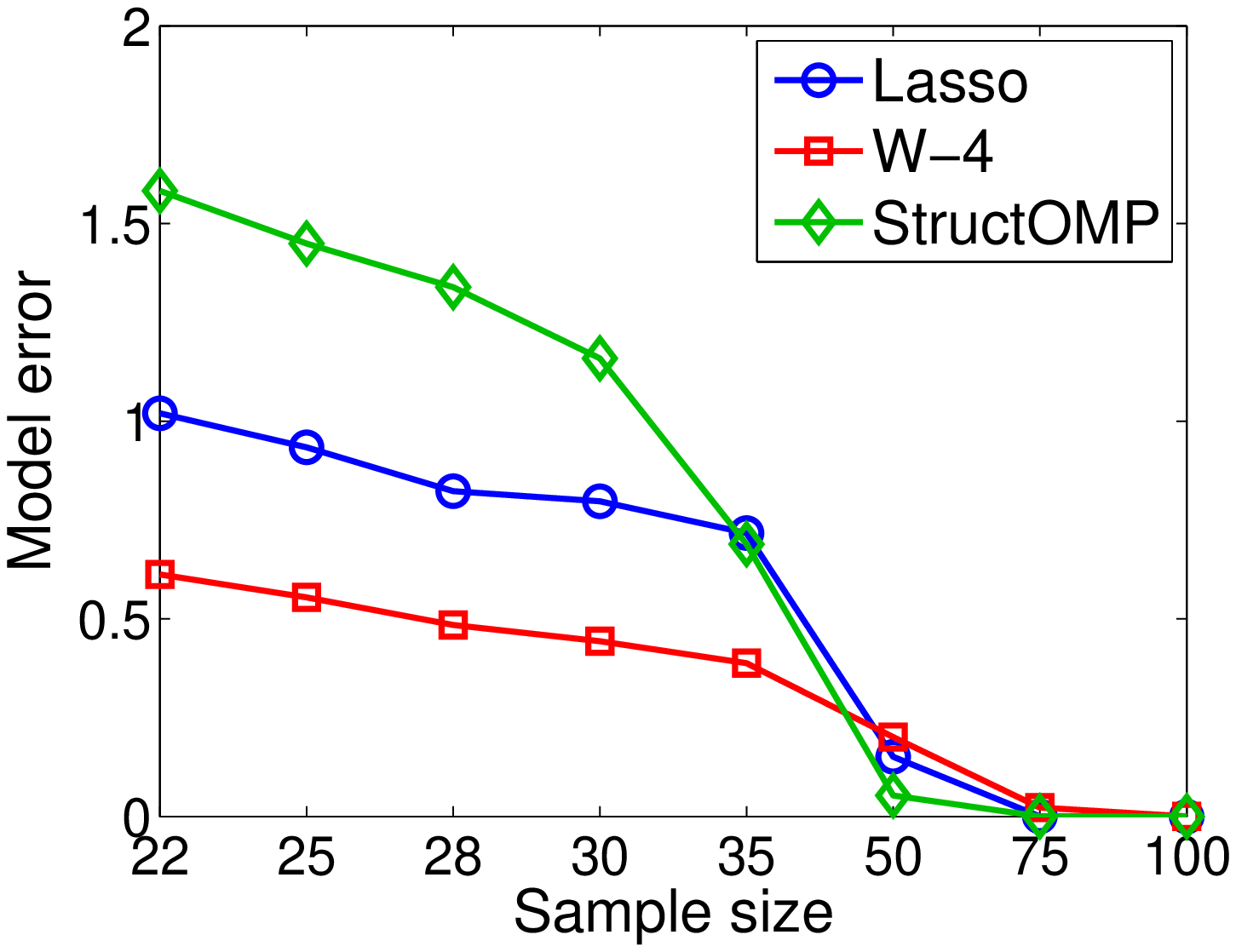} \\
    $(c)$ & $(d)$ \\
  \end{tabular}
  \caption {Comparison between StructOMP and penalty
  $\Omega(\beta|W^k)$, $k=1, \ldots, 4$, used for several polynomial
  models with random values between the roots: $(a)$ degree $1$, $(b)$
  degree $2$, $(c)$ degree $3$; $(d)$ degree $4$.}
  \label{fig:poly.unif}
\end{center}
\end{figure}

Finally, Figure \ref{fig:exp.poly.stem} displays the regression vector
found by the Lasso and the vector learned by ``W-2'' (left) and by the
Lasso and ``W-3'' (right), in a single run with sample size of $15$
and $35$, respectively. The estimated vectors (green) are superposed to the
true vector (black). Our method provides a better
estimate than the Lasso in both cases.
We found that the estimates of StructOMP are too variable for it to be
meaningful to include one of them here.

\begin{figure}[th]
  \centering
  \begin{tabular}{cc}
  \includegraphics[width=75mm,height=28.08mm]{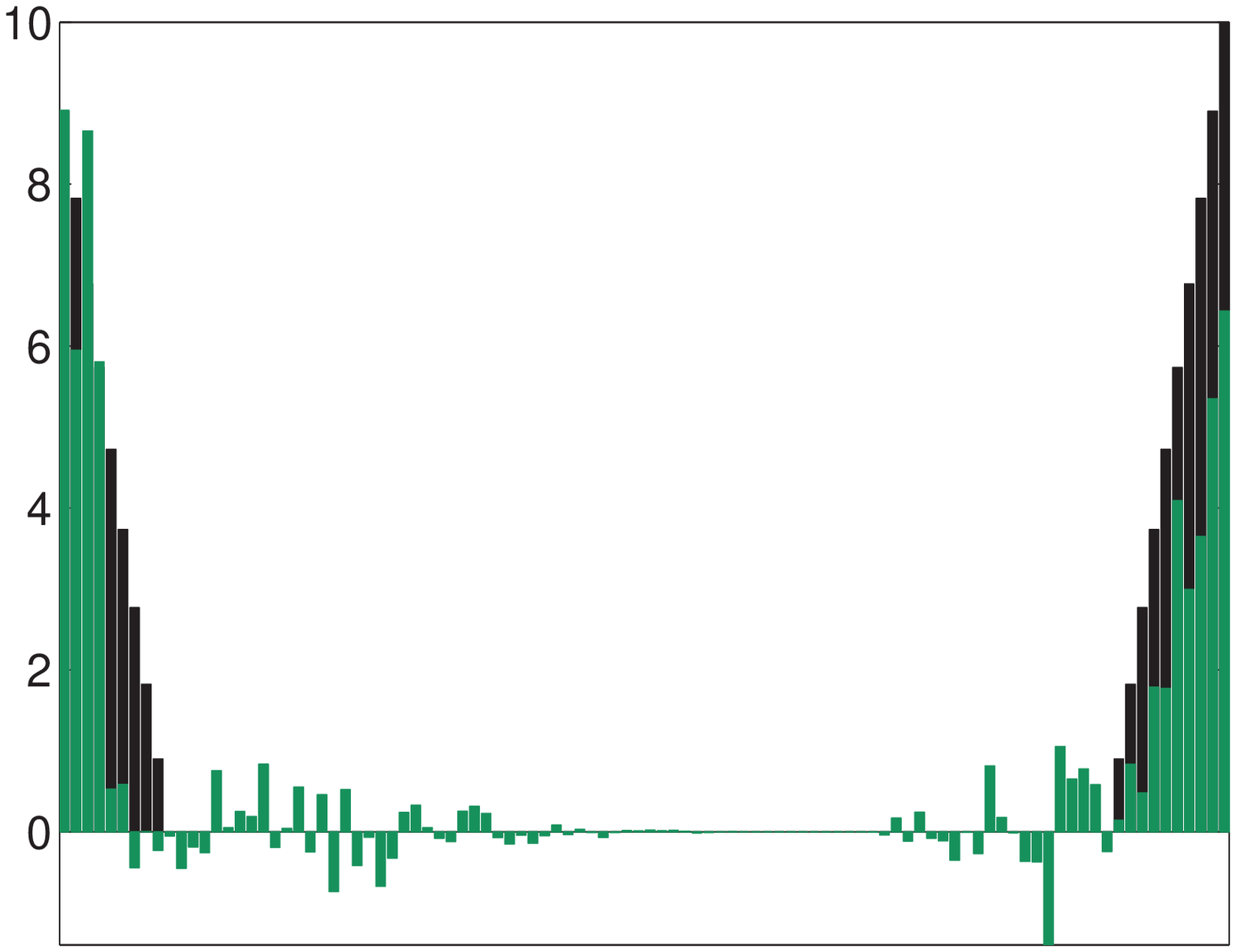} &
  \includegraphics[width=75mm,height=28.08mm]{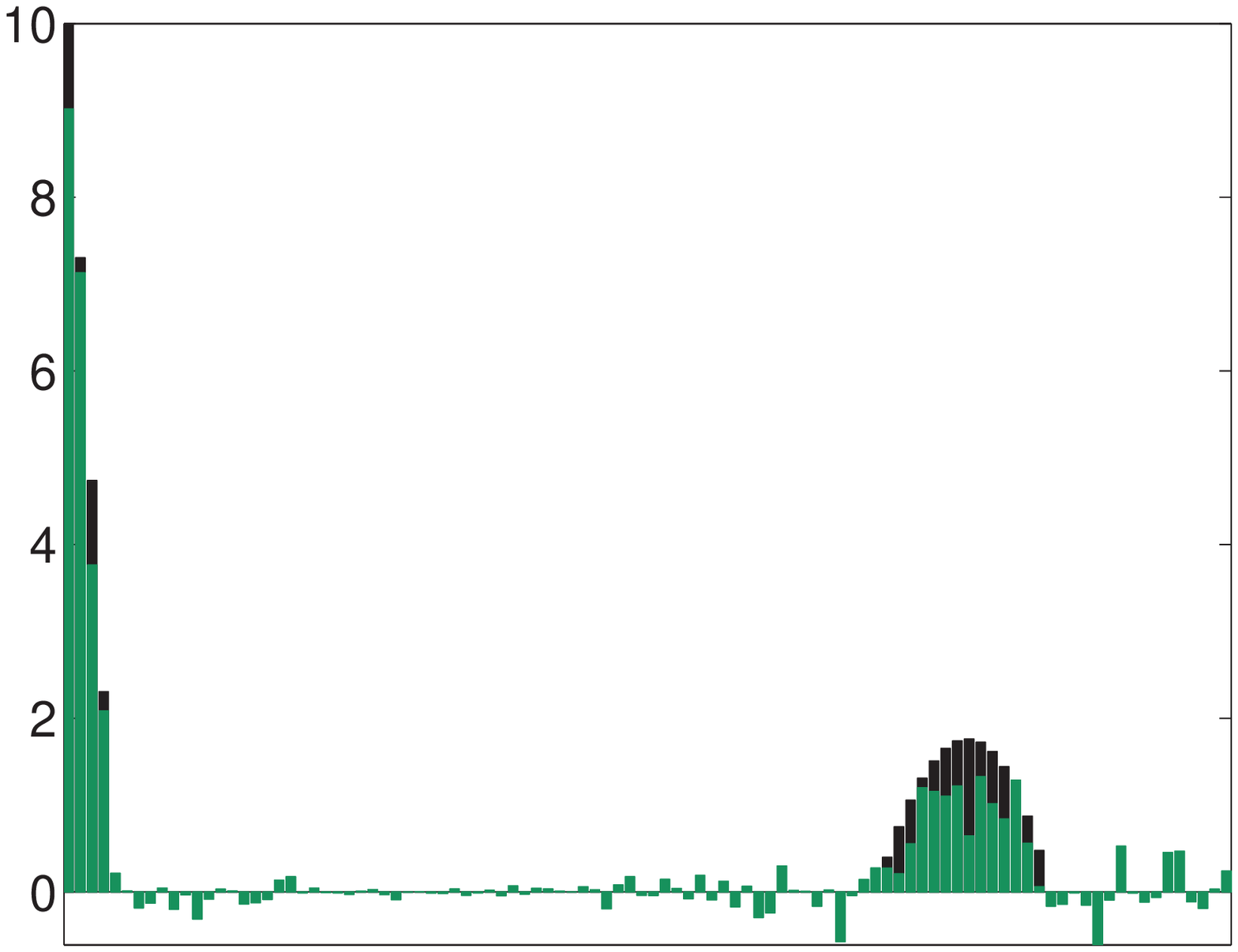} \\ 
  \includegraphics[width=75mm,height=28.08mm]{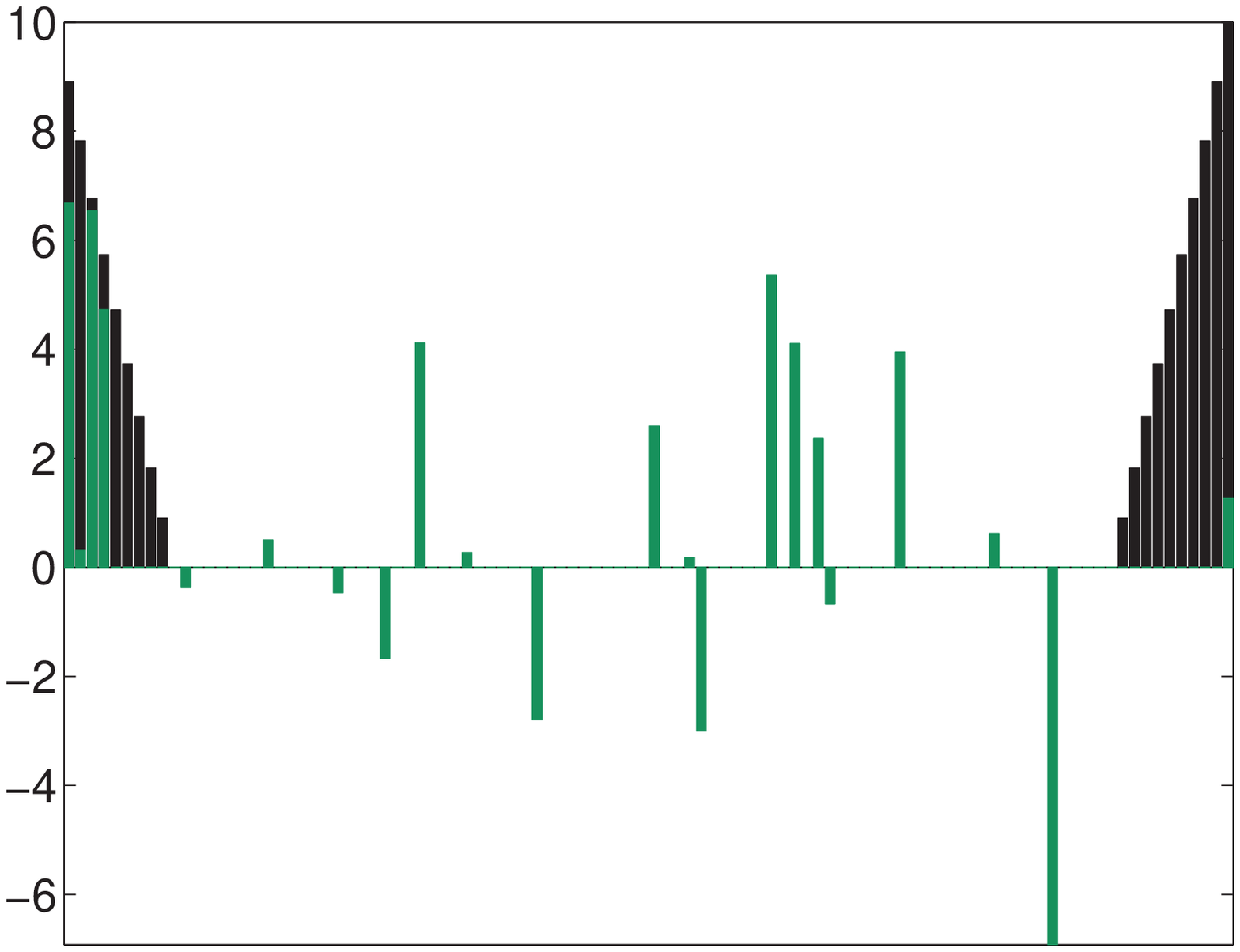} &
  \includegraphics[width=75mm,height=28.08mm]{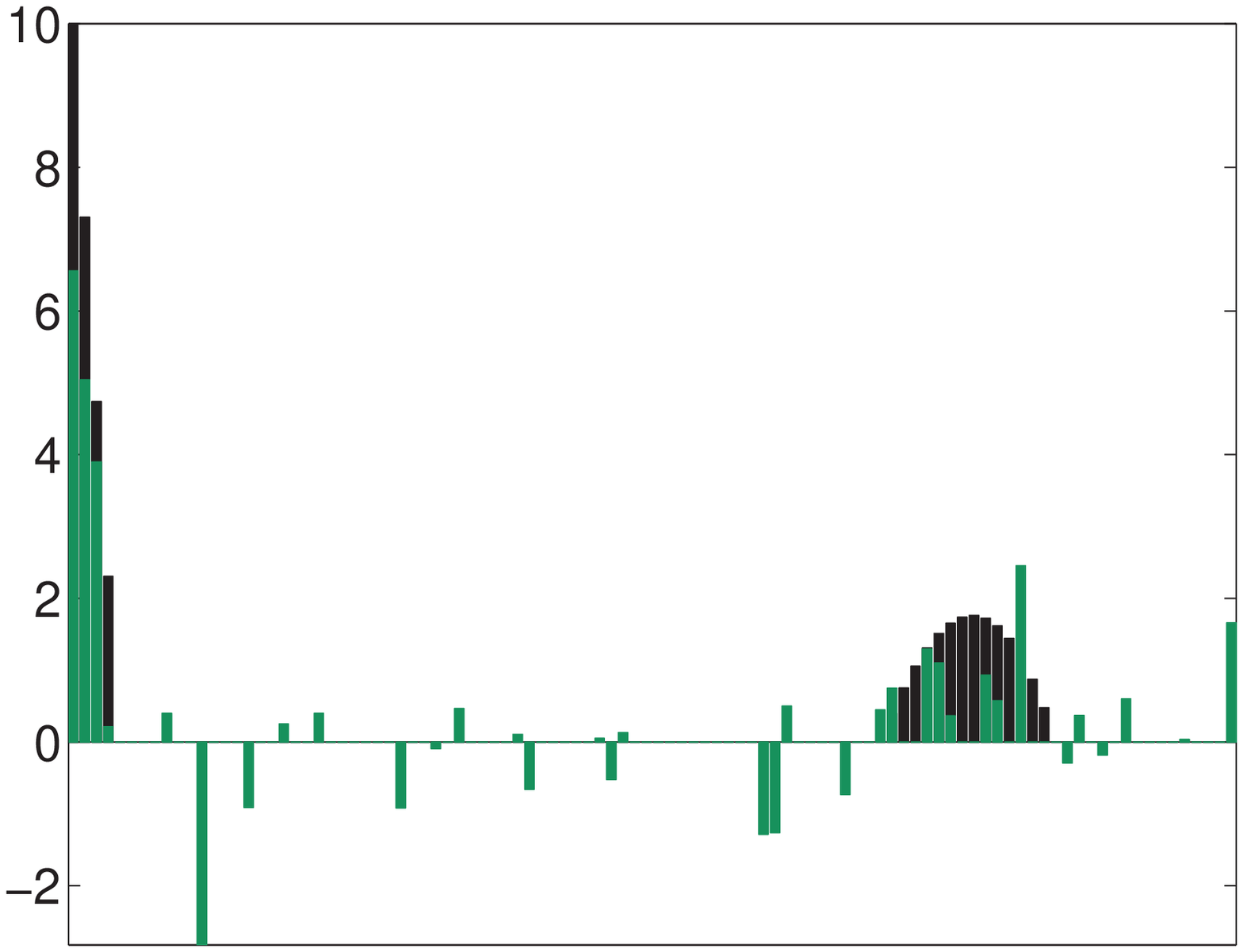} \\
  \end{tabular}
  \caption{Lasso vs. penalty $\Omega(\cdot|\La)$ for Convex (left) and Cubic (Right); see text for more information.}
  \label{fig:exp.poly.stem}
\end{figure}

\section{Conclusion}
We proposed a family of penalty functions that can be used to model
structured sparsity in linear regression. We provided
theoretical, algorithmic and computational information about this new
class of penalty functions. Our theoretical observations highlight the
generality of this framework to model structured sparsity.
An important feature of our approach is that it can deal with
richer model structures than current approaches while maintaining convexity of the penalty function.
Our practical experience indicates that these penalties perform
well numerically, improving over state of the art penalty methods for structure sparsity,
suggesting that our framework is promising for applications.

The methods developed here can be extended in different
directions. We mention here several possibilities. For example,
for any $r > 0$, it readily follows that \beq \|\beta\|_p^p =
\inf\left\{\frac{r}{r+1} \sum_{i \in \N_n}
\frac{\beta_i^2}{\lam_i} + \frac{1}{r} \lam_i^r: \lam \in \R_{++}^n\right\} 
\eeq where
$p= 2r/(r+1)$ and $\|\beta\|_p$ is the usual $\ell^p$-norm on
$\R^n$. This formula leads us to consider the same optimization
problem over a constraint set $\La$. Note that if $p \rightarrow
0$ the left hand side of the above equation converges to the
cardinality of the support of the vector $\beta$.

Problems associated with multi-task learning \cite{AEP,argyriou2010spectral} demand matrix analogs
of the results discussed here. In this regard, we propose the
following family of unitarily invariant norms on $d \times n$
matrices. Let $k = \min(d,n)$ and $\sigma(B) \in \R_{+}^k$ be the
vector formed from the singular values of $B$. When $\La$ is a
nonempty convex set which is invariant under permutations our
point of view in this paper suggests the penalty
$$
\|B\|_{\La} = \Omega(\sigma(B)|\La).
$$
The fact that this is a norm, follows from the von Neumann
characterization of unitarily invariant norms. When $\La =
\R_{++}^k$ this norm reduces to the trace norm \cite{argyriou2010spectral}.

Finally, the ideas discussed in this paper can be used in the
context of kernel learning, see \cite{bach,KolYua10,Lanckriet,feature_space,Suzuki} and references therein. Let $K_\ell$, $\ell \in \N_n$
be prescribed reproducing kernels on a set $\calX$, and $H_\ell$
the corresponding reproducing kernel Hilbert spaces with norms
$\|\cdot\|_\ell$. We consider the problem
$$ \min \left\{ \sum_{i\in \N_m} \left(y_i - \sum_{\ell\in \N_n} f_\ell(x_i)\right)^2
+ \rho \Omega^2\Big((\|f_\ell\|_\ell: \ell \in \N_n) | \La\Big) :
f_\ell \in H_{\ell}, \ell \in \N_n\right\}
$$
and note that the choice $\La = \R_{++}^n$ corresponds to multiple
kernel learning.


All the above examples deserve a detailed analysis and we hope to
provide such in future work.

\subsubsection*{Acknowledgements}
We are grateful to A. Argyriou for valuable discussions, especially
concerning the proof of Theorem \ref{thm:aa} and Theorem \ref{aa:gen}
in the appendix. We also wish to thank Luca Baldassarre, Mark
Herbster, Andreas Maurer,
Raphael Hauser, Alexandre Tsybakov and Yiming Ying for useful discussions.
This work was supported by Air Force Grant AFOSR-FA9550, EPSRC Grants
EP/D071542/1 and EP/H027203/1, NSF Grant ITR-0312113, Royal Society
International Joint Project Grant 2012/R2, as well as by the IST
Programme of the European Community, under the PASCAL Network of
Excellence, IST-2002-506778.
%
%
\appendix
\section{Appendix}
In this appendix we describe in detail a result due to J.M. Danskin, which we use in the proof of Proposition \ref{prop:0}.
\begin{definition}
Let $f$ be a real-valued function defined on an open subset $X$
of~$\R^n$ and $u \in \R^n$.~The directional derivative of $f$ at
$x \in X$ in the ``direction'' $u$ is denoted by $(D_u f)(x)$ and
is defined as
$$
(D_uf)(x) := \lim_{t \rightarrow 0} \frac{f(x+tu) - f(x)}{t}
$$
if the limit exists. When the limit is taken through nonnegative
values of $t$, we denote the corresponding right directional
derivative by $D^+_u$.
\end{definition}
Let $Y$ be a compact metric space, $F: X \times Y \rightarrow \R$ a continuous function on its domain and define the
function $f: X \rightarrow \R$ at $x \in X$ as
$$
f(x) = \min\left\{F(x,y): y \in Y \right\}.
$$
We say that $F$ is Danskin function if, for every $u \in \R^n$, the function $F'_u: X \times Y \rightarrow \R$
defined at $(x,y) \in X \times Y$ as $F'_u(x,y) = (D_u F(\cdot,y))(x)$ is continuous on $X \times Y$. Our notation
is meant to convey the fact that the directional derivative is taken relative to the first variable of $F$.
\begin{theorem}
If $X$ is an open subset of $\R^n$, $Y$ a is compact metric space, $F: X \times Y$ is a Danskin function, $u \in \R^n$ and $x \in X$, then
$$
(D^+_uf)(x) = \min \left\{F'_u(x,y): y \in Y_x\right\}
$$
where $Y_x := \{y: y \in Y ,~F(x,y) = f(x)\}$.
\label{thm:dan}
\end{theorem}
\begin{proof}
If $x \in X$, $y \in Y_x$ and $u \in \R^n$ then, for all positive $t$, sufficiently small, we have that
$$
\frac{f(x+tu) - f(x)}{t} \leq \frac{F(x+tu,y)-F(x,y)}{t}.
$$
Letting $t \rightarrow 0^+$, we get that
\beq
\limsup_{t \rightarrow 0^+} \frac{f(x+tu)-f(x)}{t} \leq \min\left\{F'_u(x,y): y \in Y_x\right\}.
\label{eq:up}
\eeq
Next, we choose a sequence $\{t_k: k \in \N\}$ of positive numbers such that $\lim_{k \rightarrow \infty} t_k = 0$ and
$${\lim_{k \rightarrow \infty}} \frac{f(x+t_ku)-f(x)}{t_k} =
\liminf_{t \rightarrow 0^+} \frac{f(x+tu)-f(x)}{t}.$$
From the definition of the function $f$, there exists a $y_k \in Y$ such that $f(x+t_k u) = F(x+t_ku,y_k)$. Since $Y$ is a compact metric space,
there is a subsequence $\{y_{k_\ell}: \ell \in \N\}$ which converges to some $y_\infty \in Y$. It readily follows
from our hypothesis that the function $f$ is continuous on $X$.
Indeed, we have, for every $x_1,x_2 \in X$, that
$$
|f(x_1)-f(x_2)| \leq \max \left\{ |F(x_1,y)-F(x_2,y)|: y \in Y\right\}.
$$
Hence we conclude that $y_\infty \in Y_x$. Moreover, we have that
$$\frac{f(x+t_ku) - f(x)}{t_k} \geq \frac{F(x+t_ku,y_k)- F(x,y_k)}{t_k}.
$$
By the mean value theorem, we conclude that there is positive number $\sigma_k < t_k$ such that the
$$\frac{f(x+t_ku) - f(x)}{t_k} \geq F'_u(x+\sigma_ku,y_k).
$$
We let $\ell \rightarrow \infty$ and use the hypothesis that $F$ is a Danskin function to conclude that
$$
\liminf_{t \rightarrow 0^+} \frac{f(x+tu)-f(x)}{t} \geq F'_u(x,y_\infty) \geq \min\left\{F'_u(x,y): y \in Y_x\right\}.
$$
Combining this inequality with \eqref{eq:up} proves the result.
\end{proof}


We note that \cite[p.~737]{Bert} describes a result which is attributed to Danskin without reference. That
result differs from the result presented above. The result in \cite[p.~737]{Bert} requires the hypothesis of convexity
on the function $F$. The theorem above and its proof is an adaptation of Theorem 1 in \cite{danskin}.

We are now ready to present the proof of Proposition \ref{prop:0}.

\vspace{.2truecm}
\noindent {\bf Proof of Proposition \ref{prop:0}}
The essential part of the proof is an application of Theorem \ref{thm:dan}.
To apply this result, we start with a $\beta \in (\R\backslash\{0\})^n$ and introduce a neighborhood of this vector defined as
$$
X(\beta) = \left\{\alpha: \alpha \in \La, \|\alpha-\beta\|_\infty < \frac{\beta_{\min}}{2}\right\},
$$
where $\beta_{\min} = \min \{|\beta_i| : i \in \N_n\}$. Theorem \ref{thm:dan} also requires us to specify a compact subset $Y(\beta)$ of
$\R^n$. We construct this set in the following way. We choose a fixed $\blam \in \La$ and
a positive $\epsilon > 0$. From these constants we
define the constants
\begin{eqnarray}
\nonumber
c(\beta) & = & \sum_{i \in \N_n} \left(\frac{(|\beta_i|+\beta_{\min}/2)^2}{\blam_i} + \blam_i\right), \\ \nonumber \\ \nonumber
a(\beta) & = &\frac{\beta_{\min}^2}{4(c(\beta)+\epsilon)},\\ \nonumber
\\ \nonumber
b(\beta) & = & \max(a(\beta),c(\beta)+\epsilon).
\end{eqnarray}
With these definitions,
we choose our compact set $Y(\beta)$ to be
$Y(\beta) = \La_{a(\beta),b(\beta)}$. To apply Theorem \ref{thm:dan}, we use the fact, for any $\alpha \in X(\beta)$, that
\beq
\label{eq:rest}
\Omega(\alpha|\La) = \min\{\Gamma(\alpha,\lam):\lam \in Y(\beta)\}.
\eeq
Let us, for the moment, assume the validity of this equation and proceed with the remaining details of the proof.
As a consequence of this equation, we conclude that there exists a vector $\lam(\beta)$ such that
$\Omega(\beta|\La) = \Gamma(\beta,\lam(\beta))$. Moreover, when $\beta \in (\R \backslash \{0\})^n$ the function $\Gamma_\beta: \R_{++}^n \rightarrow \R$, defined for
$\lambda \in \R_{++}^n$, as $\Gamma_\beta(\lambda) =
\Gamma(\beta,\lam)$ is strictly convex on its domain and so, $\lam(\beta)$ is unique.

By construction, we know, for every $\alpha \in X(\beta)$, that
$$
\max\left\{\left|\lam_i(\alpha) - \frac{a(\beta)+b(\beta)}{2}\right|: i\in \N_n \right\} \leq \frac{a(\beta)+b(\beta)}{2}.
$$
From this inequality we shall establish that $\lam(\beta)$ depends continuously on $\beta$.
To this end, we choose any sequence $\{\beta^k : k \in \N\}$ which converges to $\beta$ and from the
above inequality we conclude that the sequence of vectors $\lam(\beta^k)$ is bounded. However this sequence
can only have one cluster point, namely $\lam(\beta)$, because $\Gamma$ is continuous. Specifically,
if $\lim_{k \rightarrow \infty} \lam(\beta^k) = {\tilde \lam}$, then, for every $\lam \in
\La$, it holds that $\Gamma(\beta^k,\lam(\beta^k)) \leq \Gamma(\beta^k,\lam)$
and, passing to the limit $\Gamma(\beta,{\tilde \lam}) \leq \Gamma(\beta,\lam)$, implying that ${\tilde \lam}=\lam(\beta)$.

Likewise, equation \eqref{eq:rest} yields the formula for the partial
derivatives of $\Omega(\cdot|\La)$. Specifically, we identify $F$ and
$f$ in Theorem \ref{thm:dan} with $\Gamma$ and $\Omega(\cdot|\La)$,
respectively, and note that $$
\frac{\partial \Omega}{\partial \beta_i}(\beta|\La) = \min \left\{\frac{\partial \Gamma}{\partial \beta_i}(\beta,\lam): \lam \in \La,~\Gamma(\beta,\lam)=\Omega(\beta|\La)\right\} = \frac{\partial \Gamma}{\partial \beta_i}(\beta,\lam(\beta))  = 2 \frac{\beta_i}{\lam_i(\beta)}.
$$

Therefore, the proof will be completed after we have established equation \eqref{eq:rest}.
To this end, we note that if $\lam=(\lam_i: i \in \N_n) \in \La \backslash Y(\beta)$ then there
exists $j \in \N_n$ such that
either $\lam_j < a(\beta)$ or $\lam_j > b(\beta)$. Thus, we have, for every $\alpha \in X(\beta)$, that
$$
\Gamma(\alpha,\lam) \geq \frac{1}{2} \left(\frac{\alpha_j^2}{\lam_j} + \lam_j\right)
\geq \frac{1}{2} \min\left(\frac{\beta_{\min}^2}{4a(\beta)},b(\beta)\right)
= \frac{c(\beta) +\epsilon}{2} \geq \Omega(\alpha|\La) + \frac{\epsilon}{2}.
$$
This inequality yields equation \eqref{eq:rest}.
\qed
\vspace{.2truecm}

We end this appendix by extracting the essential features of the
convergence of the alternating algorithm as described in Section
\ref{sec:algo}. We start with two compact sets, $X \subseteq \R^n$
and $Y \subseteq \R^m$, and a strictly convex function $F: X
\times Y \rightarrow \R$. Corresponding to $F$ we introduce two
additional functions, $f: X \rightarrow \R$ and $g:Y \rightarrow
\R$ defined, for every $x \in X,y \in Y$ as
$$
f(x) = \min \{F(x,y'): y' \in Y\},~~~~~g(y) = \min\{F(x',y): x' \in X\}.
$$
Moreover, we introduce the mappings $\phi_1: Y \rightarrow X$ and $\phi_2: X \rightarrow Y$, defined, for
every $x \in X$, $y \in Y$, as
$$
\phi_1(y) = {\rm argmin} \{F(x,y): x \in X\},~~~~\phi_2(x) = {\rm argmin} \{F(x,y): y \in Y\}.
$$
\begin{lemma}
The mappings $\phi_1$ and $\phi_2$ are continuous on their respective domain.
\end{lemma}
\begin{proof}
We prove that $\phi_1$ is continuous. The same argument applies to $\phi_2$. Suppose that $\{y^k: k \in \N\}$ is a sequence
in $Y$ which converges to some point $y \in Y$. Then, since $F$ is jointly strictly convex, the sequence
$\{\phi_1(y^k): k \in \N\}$ has only one cluster point in $X$, namely $\phi_1(y)$. Indeed, if there is a subsequence $\{\phi_1(y^{k_\ell}); \ell \in \N\}$ which converges to ${\tilde x}$, then by definition, we have, for every $x \in X$, $\ell \in \N$, that $F(\phi_1(y^{k_\ell}),y^{k_\ell}) \leq F(x,y^{k_\ell})$. From this inequality it follows that $F({\tilde x},y) \leq
F(x,y)$. Consequently, we conclude that ${\tilde x} = \phi_1(y)$. Finally, since $X$ is compact, we conclude that
the $\lim_{k \rightarrow \infty} \phi_1(y^k) = \phi_1(y)$.
\end{proof}
As an immediate consequence of the lemma, we see that $f$ and $g$ are continuous on their respective domains, because,
for every $x \in X, y \in Y$, we have that $f(x) = F(x,\phi_2(x))$ and $g(y) = F(\phi_1(y),y)$.

We are now ready to define the alternating algorithm.
\begin{definition}
Choose any $y_0 \in {\rm int}(Y)$ and, for every $k \in \N$, define the iterates
$$
x^k = \phi_1(y^{k-1})
$$
and
$$y^k = \phi_2(x^k).
$$
\end{definition}
\begin{theorem}
If $F: X \times Y \rightarrow \R$ satisfies the above hypotheses
and it is differentiable on the interior of its domain, and there
are compact subsets $X_0 \subset {\rm int}(X)$, $Y_0 \subseteq
{\rm int}(Y)$ such that, for all $k \in \N$, $(x^k,y^k) \in X_0
\times Y_0$, then the sequence $\{(x^k,y^k): k \in \N\}$ converges
to the unique minimum of $F$ on its domain. \label{aa:gen}
\end{theorem}
\begin{proof}
First, we define, for every $k \in \N$, the real numbers $\theta_k=F(x^k,y^{k-1})$ and $\nu_k = F(x^k,y^k)$. We observe,
for all $k \geq 2$, that
$$
\nu_{k} \leq \theta_k \leq \nu_{k-1}.
$$
Therefore, there exists a constant $\psi$ such that $\lim_{k \rightarrow \infty} \theta_{k} = \lim_{k \rightarrow \infty} \nu_k = \psi$. Suppose, there is a subsequence $\{x^{k_\ell}: \ell \in \N\}$ such that $\lim_{\ell \rightarrow \infty} x^{k_\ell} = x$.
Then $\lim_{\ell \rightarrow \infty} \phi_2(x^{k_\ell}) = \phi_2(x) =:y$.
Observe that $\nu_k = f(x^k)$ and $\theta_{k+1} = g(y^k)$. Hence we conclude that
$$
f({x}) = g({y}) = \psi.
$$
Since $F$ is differentiable, $(x,y)$ is a stationary point of $F$
in ${\rm int}(X) \times {\rm int}(Y)$. Moreover, since $F$ is
strictly convex, it has a unique stationary point which occurs at
its global minimum.
\end{proof}

\bibliography{representer,panel,additional-bib}
\end{document}